\documentclass{article} 
\usepackage{iclr2027_conference,times}

\usepackage{amsmath,amssymb,amsthm}
\usepackage{booktabs,multirow,graphicx}
\usepackage{xcolor}
\usepackage{colortbl}
\usepackage{makecell}
\usepackage{hyperref}
\usepackage{url}
\usepackage{caption}
\usepackage{algorithm}
\usepackage{algorithmicx}
\usepackage{algpseudocode} 
\usepackage{wrapfig}

\newcommand{\ftc}{\textsc{Ftc}}
\newcommand{\rrho}{\rho}

\definecolor{best}{HTML}{DFF0D8}
\definecolor{worse}{HTML}{C0392B}
\definecolor{better}{HTML}{1E8449}

\newtheorem{proposition}{Proposition}[section]

\makeatletter\let\primitiveinput\@@input\makeatother

\title{Sequential Functional Structured Tucker Compression for Large Language Model Attentions}

\author{%
\begin{minipage}{0.96\textwidth}
\centering
\normalfont
{%
\renewcommand{\arraystretch}{1.15}%
\setlength{\tabcolsep}{10pt}%
\begin{tabular}{cccc}
\textbf{Jiangfeng Chen}$^{1,*}$ &
\textbf{Xinyu Wang}$^{2,3,*}$ &
\textbf{Tianshuo Yan}$^{4}$ &
\textbf{Hanwei Wu}$^{3,5}$
\end{tabular}\\[0.38em]
\begin{tabular}{ccc}
\textbf{Xiao-Wen Chang}$^{2}$ &
\textbf{Yang Zhang}$^{1}$ &
\textbf{Lei Ding}$^{1,\dagger}$
\end{tabular}%
}\\[0.78em]
{\small
$^{1}$University of Manitoba \quad
$^{2}$McGill University \quad
$^{3}$Simpleway\\[-0.05em]
$^{4}$The University of Hong Kong \quad
$^{5}$McMaster University
}\\[0.42em]
{\footnotesize
$^{*}$Equal contribution \qquad
$^{\dagger}$Corresponding author: \texttt{lei.ding@umanitoba.ca}
}
\end{minipage}%
}

\iclrfinalcopy
\begin{document}

\maketitle
\lhead{Preprint}

\begin{abstract}

Post-training compression of LLM attention is often formulated as
independent matrix approximation, ignoring both the shared structure
among attention projections and the representation shift introduced by
earlier compression. We propose \ftc{}, a sequential structured
compression framework that adapts the approximation to the current
compressed model while jointly exploiting the native Q/K/V head
structure under a fixed storage budget. The output projection is handled
separately to account for the changed post-attention representation.
\ftc{} requires neither fine-tuning nor gradient-based recovery.
Across seven decoder-only LLMs from 6B to 32B parameters, \ftc{} achieves
the lowest WikiText-2 perplexity among the compared methods at every
tested keep ratio on five modern GQA models, with the largest gains under
aggressive compression. The improvements transfer to downstream tasks
and remain substantial at the 32B scale.
\end{abstract}

\section{Introduction}

Large language models (LLMs) remain costly at inference time because Transformer blocks repeatedly apply large dense projection matrices, creating substantial memory and computational overhead on resource-constrained hardware. A common approach is to approximate each projection independently with low-rank factors \citep{yuan2023asvd,wang2025svdllm}. This can miss redundancy between projections: the query, key, and value projections read the same hidden input and may use overlapping feature directions. Although independent low-rank factorizations reduce the cost of each projection, they may still spend a fixed budget repeatedly representing feature directions shared across Q/K/V. A joint tensor representation can instead represent these common directions through shared factors, leaving more of the budget to capture the differences across projections and heads. Previous methods like TensorLLM and LeSTD exploit shared structure through Tucker representations \citep{gu2025tensorllm,li2026lestd}. The remaining challenge is to share useful feature structure without overly restricting the representation available to individual heads.

Sharing structure across Q/K/V is not sufficient by itself, since the features worth preserving depend on how approximation error is evaluated.Activation-aware low-rank methods use calibration activations for this purpose \citep{yuan2023asvd,wang2025svdllm}, yet the representations seen by later layers change as compression proceeds. A shared subspace fitted to the original activations may therefore become misaligned with the current compressed model. Moreover, the output projection operates on the post-attention representation, so it requires a separate approximation metric. These observations motivate a sequential framework that adapts the Q/K/V representation to the current compressed model while treating the output projection separately.

We propose FTC, a sequential \textbf{F}unctional structured \textbf{T}ucker  \textbf{C}ompression framework for attention projections. At each layer, FTC constructs an activation-induced metric from the current compressed model and uses it to fit the query, key, and value projections jointly. The factors along the model and within-head feature dimensions are shared across projections, while each head slice retains its own core coefficients. Leaving the head-slice mode uncompressed preserves head-specific coefficients and naturally supports both multi-head attention (MHA) and grouped-query attention (GQA) \citep{ainslie2023gqa}. We also balance the head slices during factor estimation to keep differences in projection magnitude from dominating the shared fit.

Sharing feature directions does not require each head to use every possible combination of them. A dense Tucker core nevertheless stores a coefficient for every pair of retained input and within-head directions. We instead allow each head slice to retain a subset of these interactions. The Tucker ranks and the total number of retained coefficients are chosen together, counting the factors, core values, and support metadata in the budget. This provides a choice between fewer shared directions with a denser core and more directions with a sparser core. We then prune the core gradually, refitting the factors and surviving coefficients between stages.

The output projection requires a different treatment because it acts on the attention representation rather than on the common input to the query, key, and value projections. After compressing these projections, we recompute the attention features. Using these updated features, we construct a calibrated target for the output projection by matching the original model's attention-branch outputs with regularization toward the pretrained weights, yielding a closed-form linear solution.. The calibrated target is then compressed under a separate activation-induced metric, and the updated block provides the inputs for compressing the next layer. 

We evaluate FTC across seven decoder-only LLMs spanning MHA and GQA architectures, with the largest gains appearing under aggressive compression. A matched per-matrix baseline helps separate the effect of the joint structured representation from that of the activation-induced metric and sequential calibration.

Our contributions are as follows:
\begin{itemize}
    \item We develop FTC, which jointly fits shared Q/K/V Tucker factors under an activation-induced metric from the current compressed model, while preserving the native MHA/GQA head structure. The output projection is calibrated separately in closed form using updated attention features and teacher branch outputs.

    \item We jointly allocate Tucker ranks and sparse-core capacity under a storage budget accounting for factors, retained coefficients, and support metadata. The resulting allocation balances subspace and sparsification errors, followed by gradual core pruning with intermediate refitting.

    \item Across seven decoder-only LLMs from 6B to 32B, FTC achieves lower WikiText-2 perplexity and higher aggregate downstream accuracy under aggressive compression on all five GQA models, without fine-tuning or gradient-based recovery. A matched per-matrix control further supports the benefit of the structured representation.
\end{itemize}

\vspace{-0.3cm}

\section{Related Work}
\label{sec:related_work}

\paragraph{Activation-aware low-rank compression.}
Recent low-rank compression methods extend beyond weight-space reconstruction by incorporating model inputs or parameter importance into the approximation objective. Fisher-weighted SVD uses parameter sensitivity
\citep{hsu2022fisherweighted}, while ASVD and SVD-LLM exploit activation
information to improve low-rank approximation
\citep{yuan2023asvd,wang2025svdllm}. Later work further considers rank
allocation, truncation-error compensation, and numerical stability
\citep{wang2025svdllmv2,li2025adasvd,parkina2025coala}. \ftc{}  differs by
constructing its activation-induced metric from the current compressed
model and using it to fit a joint head-wise Q/K/V representation rather
than independent projection matrices.

\paragraph{Shared and tensorized representations.}
Shared and tensorized compression methods seek to exploit redundancy across
projections or layers through common low-dimensional structure. Basis
Sharing uses shared low-rank bases with layer-specific coefficients
\citep{wang2025basissharing}, while TensorLLM and LeSTD model attention
projections with Tucker-based shared factors
\citep{gu2025tensorllm,li2026lestd}. Other work explores tensor-train and
sparse tensor representations for model compression
\citep{xu2023tensorgpt,solgi2025saten}, as well as joint Q/K/V
factorization \citep{wang2025qsvd}. Recent studies also indicate that overly
restrictive shared subspaces may not capture the heterogeneous structure of
modern LLM representations \citep{zagitov2026rethinking}. FTC addresses this issue by sharing Tucker factors across Q/K/V and retaining head-specific core coefficients, balancing common structure with head-specific flexibility.

\paragraph{Sequential correction and sparse allocation.}
Sequential compression and error correction have been explored in prior
work. SparseGPT performs layerwise activation-aware pruning
\citep{frantar2023sparsegpt}, while MPIFA and AA-SVD introduce reconstruction
or refinement steps to reduce accumulated compression error
\citep{zhao2025pivoting,sinha2026aasvd}. FTC instead uses the current student
activations to define the activation-induced metric for the joint Q/K/V tensor
approximation, and compresses the output projection separately using the updated attention
representation.

Pruning methods such as Wanda and SliceGPT remove individual weights or structured hidden dimensions \citep{sun2024wanda,ashkboos2024slicegpt}, while LoRAP assigns low-rank approximation and structured pruning to different Transformer sublayers \citep{li2024lorap}.  Global low-rank-plus-sparse allocation has also been studied across layers \citep{zhou2026globalrank}.  In contrast, FTC couples the Tucker factor ranks and the number of retained latent-core coefficients within each attention layer.  Core sparsity selects rank-one interactions between shared directions rather than producing zeros in the reconstructed weight matrices; practical acceleration therefore requires a kernel that exploits the latent representation.  Hardware-supported patterns such as 2:4 provide one possible deployment route \citep{mishra2021sparse}.

\section{Function-Driven Structured Compression}
We compress the attention blocks sequentially using representations generated by the current compressed model. Figure~\ref{fig:workflow} summarizes the overall framework, and the following subsections develop its activation-induced metric, structured Q/K/V approximation, and output-projection calibration.

\begin{figure}[htbp]
    \centering
    \includegraphics[width=1\textwidth]{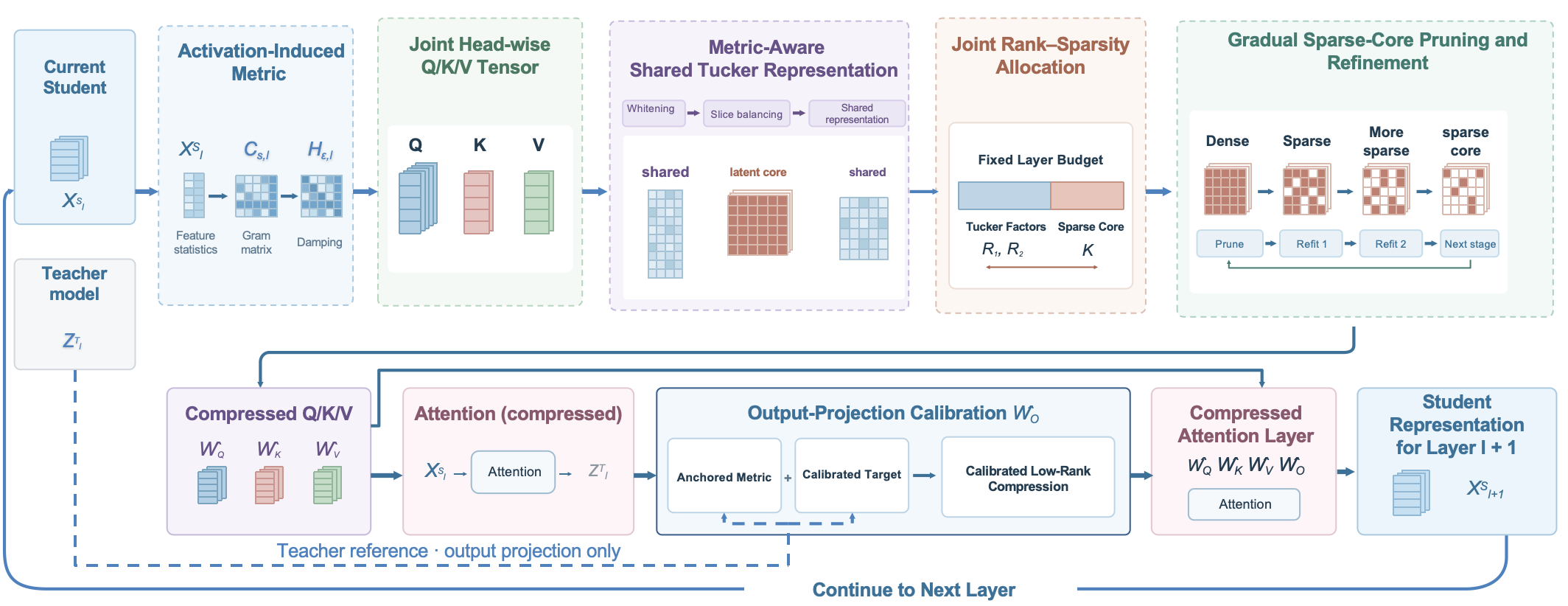}
    \caption{Overview of FTC. At each layer, FTC builds an activation-induced metric from the current student representations, jointly compresses the Q/K/V projections with a shared Tucker representation, and allocates the layer budget between Tucker ranks and sparse-core capacity. The core is then refined progressively, while \(W_O\) is calibrated and compressed separately using the updated attention representation. The resulting student representation is passed to the next layer, forming a sequential compression process.}
    \label{fig:workflow}
\end{figure}

\vspace{-0.3cm}

\subsection{Preliminaries and Problem Formulation}
\label{sec:preliminaries}

Consider the $l$-th Transformer block and let
$X_{l}^{\mathrm{S}}\in\mathbb{R}^{N\times d_{\mathrm{model}}}$ 
denote the hidden representations produced by the current compressed student model, where the superscript \(\mathrm{S}\) denotes the student stream and \(N\) is the number of tokens. Since the preceding layers have already been compressed,
$X_l^{\mathrm{S}}$ reflects the representation encountered by the
current student at layer $l$.

Let $h_{q,l}$ and $h_{kv,l}$ denote the numbers of query and key/value heads, respectively. We stack the corresponding head-wise Q/K/V projection matrices along the third mode to form the joint tensor
\begin{equation*}
	\mathcal{W}_{l} \in \mathbb{R}^{d_{\mathrm{model}}
	\times d_{\mathrm{head}}
	\times m_{l}}, \qquad m_{l}=h_{q,l}+2h_{kv,l}.
\end{equation*}
The third mode contains all query-head slices together with the native
key- and value-head slices. Hence, in GQA and MQA, shared K/V heads are
included only once rather than being replicated to match the number of
query heads. This joint representation exposes the shared feature structure across the Q/K/V projections while retaining their native head organization.

Since all Q/K/V projections operate on the same hidden representation, we compress them jointly while sharing the first two Tucker factors \citep{kolda2009tensor}:
\begin{equation*}
	\widehat{\mathcal{W}}_{l}
	= \mathcal{G}_{l} \times_1 U_{1,l} \times_2 U_{2,l},
\end{equation*}
where $U_{1,l}\in\mathbb{R}^{d_{\mathrm{model}}\times R_{1,l}}$, $U_{2,l}\in\mathbb{R}^{d_{\mathrm{head}}\times R_{2,l}}$, and 
$\mathcal{G}_{l}\in \mathbb{R}^{R_{1,l}\times R_{2,l}\times m_{l}}$ is a latent core tensor. Sharing $U_{1,l}$ and $U_{2,l}$ exploits common
feature subspaces across Q/K/V, while leaving the third mode uncompressed.   Moreover, the output projection $W_{O,l}$ is compressed separately due to its
distinct input representation; implementation details are provided in Appendix~\ref{app:implementation}.

For Q/K/V, we do not directly regress to teacher outputs. Instead, we
measure the effect of a weight perturbation on the current student
activation stream. For a projection perturbation $\Delta W_l$, 
\begin{equation*}
	\frac{1}{N} \left\| X_{l}^{\mathrm S}\Delta W_{l} \right\|_{F}^{2} = \operatorname{tr}
    \left(
        \Delta W_l^\top
        C_{s,l}
        \Delta W_l
    \right),
	\qquad C_{s,l} = \frac{1}{N}
	(X_{l}^{\mathrm{S}})^\top X_{l}^{\mathrm{S}},
\end{equation*}
where $C_{s,l}$ is the empirical second-moment matrix of the current student representations. 
We therefore define the regularized activation-induced metric
\begin{equation*}
	H_{\varepsilon,l} = C_{s,l}
	+ \varepsilon_{l} \bar{t}_{l}I, \qquad \varepsilon_{l}>0, \qquad  \bar{t}_{l} = \operatorname{tr}(C_{s,l})/ d_{\mathrm{model}},
\end{equation*}
where $\bar{t}_{l}$ is the mean diagonal value of $C_{s,l}$, which normalizes
the damping scale across layers, and $\varepsilon_{l}$ controls the
regularization strength.

For the joint Q/K/V tensor, the resulting structured approximation problem takes the form
\begin{equation}
	\min_{\widehat{\mathcal{W}}_{l}\in\mathcal{M}_{l}(B_l)}
	\left\| \widehat{\mathcal{W}}_{l}-\mathcal{W}_{l}
	\right\|_{H_{\varepsilon,l}}^{2},
	\label{eq:functional_problem}
\end{equation}
where $\mathcal{M}_{l}(B_l)$ 
is the family of such structured representations 
whose total storage cost satisfies a layerwise compression budget $B_l$, including the shared factors, retained core values, and associated index metadata. Since $H_{\varepsilon,l}$ is constructed from the current student
representations, the approximation metric is updated sequentially as compression proceeds. Equation~\eqref{eq:functional_problem} defines the activation-induced
approximation objective. We first convert this objective exactly into a
Euclidean reconstruction problem by metric whitening. We then introduce
slice balancing to prevent high-energy Q/K/V head projections from
dominating the shared Tucker factors.

\subsection{Sequential Activation-Induced Tensor Approximation}
\label{sec:sequential_tensor_approximation}

We now transform the activation-metric problem in
Eq.~\eqref{eq:functional_problem} into a standard Euclidean tensor approximation. Since $H_{\varepsilon,l}\succ 0$, let
$H_{\varepsilon,l} = L_{l}^{\top}L_{l}$ be its Cholesky factorization.
For any tensor
$\mathcal A\in \mathbb{R}^{d_{\mathrm{model}}\times d_{\mathrm{head}}\times m_{l}}$, the weighted norm satisfies
\begin{equation*}
	\|\mathcal{A}\|_{H_{\varepsilon,l}}^2
	= \|\mathcal{A}\times_1 L_{l}\|_{F}^{2}.
\end{equation*}
We therefore define the  metric-whitened tensor $\widetilde{\mathcal{W}}_{l}
	=\mathcal{W}_{l}\times_1 L_{l}.$
For the Tucker representation $\widehat{\mathcal{W}}_{l}
	= \mathcal{G}_{l}
	\times_1 U_{1,l}
	\times_2 U_{2,l},$ 
setting $\widetilde{U}_{1,l}=L_{l}U_{1,l}$ gives $\widehat{\mathcal{W}}_{l}\times_1 L_{l}
	= \mathcal{G}_{l}
	\times_1\widetilde{U}_{1,l}
	\times_2 U_{2,l}.$

Hence, Eq.~\eqref{eq:functional_problem} is equivalently written as
\begin{equation*}
	\min \left\|
	\widetilde{\mathcal{W}}_{l} - \mathcal{G}_{l}
	\times_1\widetilde{U}_{1,l}
	\times_2 U_{2,l}
	\right\|_{F}^{2},
\end{equation*}
with the layerwise storage budget $B_l$. After compression, the mode-1 factor is mapped back by
$U_{1,l}=L_l^{-1}\widetilde U_{1,l}$. Thus, metric whitening exactly converts the activation-induced objective into a Euclidean tensor approximation problem.
The following subsection introduces a separate slice-level reweighting
to balance the contributions of Q/K/V head projections with different
magnitudes.

\subsection{Shared Tucker Representation with Slice Balancing}
\label{sec:shared_tensor_representation}

After metric whitening, the activation-induced objective becomes a
Euclidean reconstruction problem. We further balance the Q/K/V head
slices when estimating the shared Tucker factors, so that high-energy
slices do not dominate the shared subspaces. Specifically, let
\begin{equation*}
	\widetilde{W}_{l,j}:= \widetilde{\mathcal{W}}_l(:,:,j),
	\qquad
	s_{j,l} := \|\widetilde{W}_{l,j}\|_{F},
	\qquad
	w_{j,l} := s_{j,l}^{-2},
\end{equation*}
for $j=1,\ldots,m_{l}$. We therefore consider 
\begin{equation}
	\min_{\mathcal{G}_{l},\widetilde{U}_{1,l},U_{2,l}}
	\sum_{j=1}^{m_l} w_{j,l}
	\left\|\widetilde{W}_{l,j}-\widetilde{U}_{1,l}\mathcal{G}_l(:,:,j)U_{2,l}^{\top}
	\right\|_{F}^{2}, \quad \text{s.t.}  \quad 
    \widetilde{U}_{1,l}^{\top}\widetilde{U}_{1,l}=I,
	\qquad U_{2,l}^{\top}U_{2,l}=I.
    \label{eq:slice_normalized_tucker}
\end{equation}
Since $w_{j,l}=\|\widetilde{W}_{l,j}\|_{F}^{-2}$, each term measures the
reconstruction error relative to the energy of the corresponding slice.
Thus, each Q/K/V slice contributes according
to its relative approximation error rather than its absolute magnitude.

To characterize the shared model-dimension subspace, we temporarily omit the mode-$2$ compression and consider the reduced problem
\begin{equation}
	\min_{\substack{
			B_{1,l},\ldots,B_{m_l,l}}}
	\sum_{j=1}^{m_l} w_{j,l}
	\left\| \widetilde{W}_{l,j}	-\widetilde{U}_{l} B_{j,l}
	\right\|_{F}^{2}, \quad \text{s.t.} \quad \widetilde{U}_{l}^{\top}\widetilde{U}_{l}=I.
 \end{equation}
For fixed \(\widetilde U_l\), the optimal coefficient matrices are $B_{j,l}^{\star}= \widetilde{U}_l^{\top}\widetilde{W}_{l,j}.$
Consequently, the optimal rank-$R_{1,l}$ shared column subspace is spanned by the leading left singular vectors of
\begin{equation*}
	\widetilde{W}_{\mathrm{cat},l}
	= \begin{bmatrix}
		\sqrt{w_{1,l}}\widetilde{W}_{l,1}&\cdots&
		\sqrt{w_{m_l,l}}\widetilde{W}_{l,m_l}
	\end{bmatrix}.
\end{equation*}
Hence, the shared mode-$1$ factor captures the dominant directions jointly
supported by the normalized Q/K/V head projections. A formal proof is given
in Appendix~\ref{app:weighted_shared_subspace}.

We extend this shared-subspace principle to the first two modes through the Tucker representation
\begin{equation*}
	\widehat{\widetilde{\mathcal{W}}}_{l}
	= \mathcal{G}_{l}
	\times_1\widetilde{U}_{1,l}
	\times_2 U_{2,l},
\end{equation*}
while leaving the third mode uncompressed. Thus, the first two factors are
shared across Q/K/V head slices, whereas each core slice retains its own
coefficients. The next subsection jointly allocates the layerwise
storage budget between the Tucker ranks and the retained sparse-core
coefficients.

\subsection{Budget-Coupled Rank--Sparsity Allocation}
\label{sec:budget_rank_sparsity}

We allocate the layerwise storage budget among the shared Tucker factors,
the retained latent-core coefficients, and their support metadata. For a candidate rank pair
$\mathbf R_l=(R_{1,l},R_{2,l})$, define 
$$P_{\mathrm{fac}}(\mathbf{R}_{l})
	= d_{\mathrm{model}}R_{1,l}+ d_{\mathrm{head}}R_{2,l}, \qquad 
    N_{l}(\mathbf{R}_{l})
    = R_{1,l}R_{2,l}m_{l}.$$  
The sparse-core support is stored as a bit-packed binary mask
$\mathcal M_l\in\{0,1\}^{R_{1,l}\times R_{2,l}\times m_l}$ with
$\|\mathcal M_l\|_0=K_l$, requiring one bit per latent-core location. Only the $K_l$ surviving core coefficients are stored, in packed form.
If each stored value uses $b_v$ bits, the total storage cost is
\begin{equation*}
    S^{\mathrm{bit}}_{l}(\mathbf R_l,K_l)
    = b_v\left[P_{\mathrm{fac}}(\mathbf R_l)+K_l \right]
    + N_l(\mathbf{R}_l).
\end{equation*}
Let $B_l=B_{l}^{\mathrm{bit}}/b_{v}$ denote the layer budget in
$b_{v}$-bit scalar equivalents. The largest admissible core cardinality is
\begin{equation*}
    K_l(\mathbf{R}_{l})
    = \min\left\{  N_l(\mathbf{R}_l),
    \left\lfloor B_{l}-
    P_{\mathrm{fac}}(\mathbf{R}_{l})
    - \frac{N_l(\mathbf{R}_{l})}{b_v}\right\rfloor
    \right\},  \qquad 
    P_{\mathrm{fac}}(\mathbf{R}_{l})
    +\frac{N_l(\mathbf{R}_{l})}{b_v}
    \leq B_{l}.
\end{equation*}


To evaluate candidate rank pairs, we absorb the slice-balancing weights
into the whitened tensor,
\begin{equation*}
	\overline{\mathcal{W}}_l(:,:,j)
	= \sqrt{w_{j,l}}\,
	\widetilde{\mathcal{W}}_l(:,:,j),
	\qquad
	j=1,\ldots,m_{l}.
\end{equation*}

For fixed orthonormal Tucker factors, the projection core is  $\overline{\mathcal{G}}_{l}^{\star}=
	\overline{\mathcal{W}}_{l}
	\times_1 \widetilde{U}_{1,l}^{\top}
	\times_2 U_{2,l}^{\top},$
with corresponding subspace approximation error 
\begin{equation*}
	E_{\mathrm{sub}}(\mathbf{R}_{l})
	= \left\| \overline{\mathcal{W}}_{l}
	- \overline{\mathcal{G}}_{l}^{\star}
	\times_1 \widetilde U_{1,l}
	\times_2 U_{2,l} \right\|_{F}^{2}.
\end{equation*}

Given a core cardinality $K_l$, let
$\mathcal P_{K_l}(\overline{\mathcal G}_l^\star)$ retain the $K_l$ largest-magnitude coefficients of the projection core. The corresponding one-step sparsification error is
\begin{equation*}
	D_{\mathrm{sp}}(\mathbf{R}_l,K_l)
	=\left\|\overline{\mathcal{G}}_l^{\star}
	-\mathcal{P}_{K_l}
	\left(\overline{\mathcal{G}}_l^{\star}\right)
	\right\|_{F}^{2}.
\end{equation*}

For fixed orthonormal Tucker factors and their projection core, the subspace residual is orthogonal to the represented Tucker subspace. Consequently, the error further decomposes
exactly as
\begin{equation}
	E_{\mathrm{struct}}(\mathbf{R}_l,K_l)
	= E_{\mathrm{sub}}(\mathbf{R}_l)
	+ D_{\mathrm{sp}}(\mathbf{R}_l,K_l).
	\label{eq:rank_sparsity_decomposition}
\end{equation}
This decomposition exposes the budget trade-off: increasing the Tucker ranks
 reduces the subspace approximation error, but consumes more storage
in the shared factors and support mask, potentially leaving fewer coefficients
for the core. We therefore optimize the Tucker ranks and core cardinality
jointly under the layerwise storage budget. The resulting problem admits the
following reduction:


\begin{proposition}[Budget-coupled rank--sparsity allocation]
\label{prop:budget_allocation}
For each feasible rank pair $\mathbf R_l\in\mathcal R_l$, let
$K_l(\mathbf{R}_{l})$ be the largest core cardinality allowed by the storage
budget. Since $D_{\mathrm{sp}}(\mathbf{R}_{l},K_{l})$ is non-increasing in $K_l$,
the joint surrogate problem
\[
\min_{\mathbf{R}_{l},K_{l}} \;
E_{\mathrm{sub}}(\mathbf{R}_l) +
\beta D_{\mathrm{sp}}(\mathbf{R}_{l},K_{l}), \qquad \beta\in(0,1],
\]
subject to the layerwise storage budget, reduces to
\[
\mathbf{R}_{l}^\star =
\arg\min_{\mathbf{R}_{l}\in\mathcal{R}_{l}} E_{\mathrm{sub}}(\mathbf{R}_{l}) + \beta D_{\mathrm{sp}}
\left(\mathbf{R}_{l},K_l(\mathbf{R}_{l})\right),
\qquad K_{l}^\star=K_{l}(\mathbf{R}_{l}^\star).
\]

\end{proposition}

\vspace{-0.1cm}

Here, $\beta$ discounts the one-step sparsification cost to account for the
subsequent prune--refit refinement. After allocation, the selected core is
gradually pruned to $K_l^\star$ with support-constrained refitting; details
are given in Appendix~\ref{app:rank_sparsity}.

\subsection{Separate Compression of the Output Projection}
\label{sec:output_projection}

The output projection is treated separately since it operates on the
post-attention representation rather than on the block input
$X_l^{\mathrm S}$ shared by Q/K/V. After compressing Q/K/V, we recompute the
student attention representation $Z_l^{\mathrm S}$ and use it to construct a
dedicated activation-induced metric for $W_O$.

When output-projection calibration is enabled, we additionally use the
corresponding teacher representation $Z_l^{\mathrm T}$ to define a target
adapted to the compressed attention features. Specifically, we solve
\[
\min_{W} \frac{1}{N}
\left\| Z_{l}^{\mathrm{S}}W-Z_{l}^{\mathrm{T}}W_{O,l} \right\|_{F}^{2}
+ \mu_{o} \left\| W-W_{O,l} \right\|_{F}^{2},
\]
which balances recovery of the teacher attention-branch output against
deviation from the pretrained output projection. This quadratic problem has
a closed-form solution, denoted by $W_{O,l}^{\mathrm{cal}}$. When calibration
is disabled, we simply use the original $W_{O,l}$ as the compression target.

After replacing $W_{O,l}$ with its compressed counterpart, the updated block
produces $X_{l+1}^{\mathrm S}$, which is used to construct the metric for the
next layer. The procedure therefore continues sequentially through the model.
Full calibration and implementation details are provided in
Appendix~\ref{app:implementation}.

\section{Experiments}
\label{sec:experiments}
We evaluate our \ftc{} across decoder-only LLMs of different scales and attention
architectures. We report perplexity and downstream performance, followed by
ablations and analyses of the main components. Below we first introduce the models, baselines and protocol of our experiments:

\textbf{Models. } 
We evaluate on two MHA models (GPT-J-6B, Llama-2-13B) and five modern GQA models (Mistral-7B-v0.3, Llama-3-8B, Qwen2.5-7B, Qwen3-8B-Base with QK-norm, and Qwen3-32B). Among them, Qwen3-32B is further used to test scale and architectural generalization.

\textbf{Baselines.}
We compare \ftc{} with five published baselines: \textbf{enh-LeSTD}, our GQA-capable port of LeSTD using the same head-concat tensorization \citep{li2026lestd};
\textbf{TLLM-gen}, a TensorLLM-style data-free dense Tucker baseline \citep{gu2025tensorllm};
\textbf{MPIFA} \citep{zhao2025pivoting}; \textbf{SVD-LLM} \citep{wang2025svdllm}; and \textbf{Basis Sharing} \citep{wang2025basissharing}.
\textbf{Seq-SVD} is our per-matrix control under the same functional metric and sequential calibration as \ftc{}, rather than a published baseline. 

\textbf{Protocol.} \ftc{} uses 256 WikiText-2 calibration sequences of 2{,}048 tokens
(seed 42), while data-dependent baselines follow their published
128-sequence configurations; Table~\ref{tab:calibsize} quantifies this
difference. All methods use the same evaluation windows and
parameter-count accounting. Following LeSTD, $\rrho$ denotes the  retained storage ratio; smaller $\rrho$ means stronger
compression. The detailed  storage accounting is provided in
Appendix~\ref{app:accounting}.

All main-table \ftc{} results use one fixed recipe across models and ratios:
the budget-allocation parameter $\beta=0.5$, the teacher-Gram anchoring
parameter $a=0.5$, no QKV teacher map, a closed-form teacher map for
$W_O$ with regularization parameter $\mu=1$, the gradual core-pruning rate $\alpha=0.02$, two refits, 256 calibration windows, and
slice-wise normalization. No fine-tuning or gradient-based recovery is used.

\subsection{Perplexity across three model generations}
\label{sec:main-results}

We first evaluate model-level degradation using WikiText-2 perplexity across
all seven models and compression ratios. Table~\ref{tab:main} reports two
representative operating points, while Figure~\ref{fig:sweep} shows the full
compression sweep.

\begin{table}[h]
\centering
\caption{WikiText-2 perplexity at $\rrho\in\{0.6,0.2\}$ (dense ppl in parentheses; best per row in bold). All methods use the same realised compression ratio. Here, ``enh-L'', ``TLLM'', and ``SVD'' abbreviate enh-LeSTD,
TLLM-gen, and SVD-LLM, respectively.}
\label{tab:main}
\setlength{\tabcolsep}{2.2pt}
\renewcommand{\arraystretch}{0.90}
\scriptsize
\begin{tabular}{@{}lcccccccccccccc@{}}
\toprule
& \multicolumn{7}{c}{$\rrho=0.6$}
& \multicolumn{7}{c}{$\rrho=0.2$} \\
\cmidrule(lr){2-8}
\cmidrule(l){9-15}
Model
& \ftc{} & Seq-SVD & enh-L & TLLM & MPIFA & SVD & Basis
& \ftc{} & Seq-SVD & enh-L & TLLM & MPIFA & SVD & Basis \\
\midrule
GPT-J-6B (8.86)
& \cellcolor{best}\textbf{9.30} & 9.47 & 9.93 & 9.68 & 9.46 & 10.26 & 10.16
& \cellcolor{best}\textbf{13.60} & 14.84 & 17.47 & 19.22 & 17.24 & 25.04 & 23.02 \\

Llama-2-13B (4.89)
& \cellcolor{best}\textbf{5.11} & 5.27 & 5.13 & 5.52 & 5.26 & 5.68 & 5.57
& \cellcolor{best}\textbf{7.52} & 9.00 & 9.79 & 13.01 & 10.21 & 14.20 & 12.77 \\

Mistral-7B (5.32)
& \cellcolor{best}\textbf{5.44} & 5.69 & 5.58 & 6.07 & 5.63 & 5.92 & 5.75
& \cellcolor{best}\textbf{7.33} & 12.10 & 146.31 & 18.41 & 15.92 & 25.40 & 16.18 \\

Llama-3-8B (6.14)
& \cellcolor{best}\textbf{6.64} & 7.03 & 6.82 & 8.70 & 7.26 & 8.02 & 7.60
& \cellcolor{best}\textbf{10.44} & 19.46 & 20.84 & 29.85 & 27.02 & 48.09 & 35.60 \\

Qwen2.5-7B (6.85)
& \cellcolor{best}\textbf{7.20} & 7.98 & 7.76 & 34.10 & 7.84 & 8.56 & 7.82
& \cellcolor{best}\textbf{12.55} & 150.93 & 106.59 & 188.22 & 7791 & 255.23 & 140.65 \\

Qwen3-8B (7.00)
& \cellcolor{best}\textbf{7.36} & 7.65 & 8.20 & 9.12 & 8.17 & 9.24 & 9.18
& \cellcolor{best}\textbf{10.70} & 23.66 & 27.08 & 45.45 & 22.23 & 75.77 & 46.68 \\

Qwen3-32B (7.61)
& \cellcolor{best}\textbf{7.77} & 7.98 & 7.84 & 8.56 & 8.40 & 9.23 & 8.70
& \cellcolor{best}\textbf{9.22} & 22.97 & 15.23 & 29.71 & 26.75 & 66.29 & 29.63 \\ 
\bottomrule
\end{tabular}
\end{table}

\begin{figure}[h]
\centering
\includegraphics[width=\linewidth]{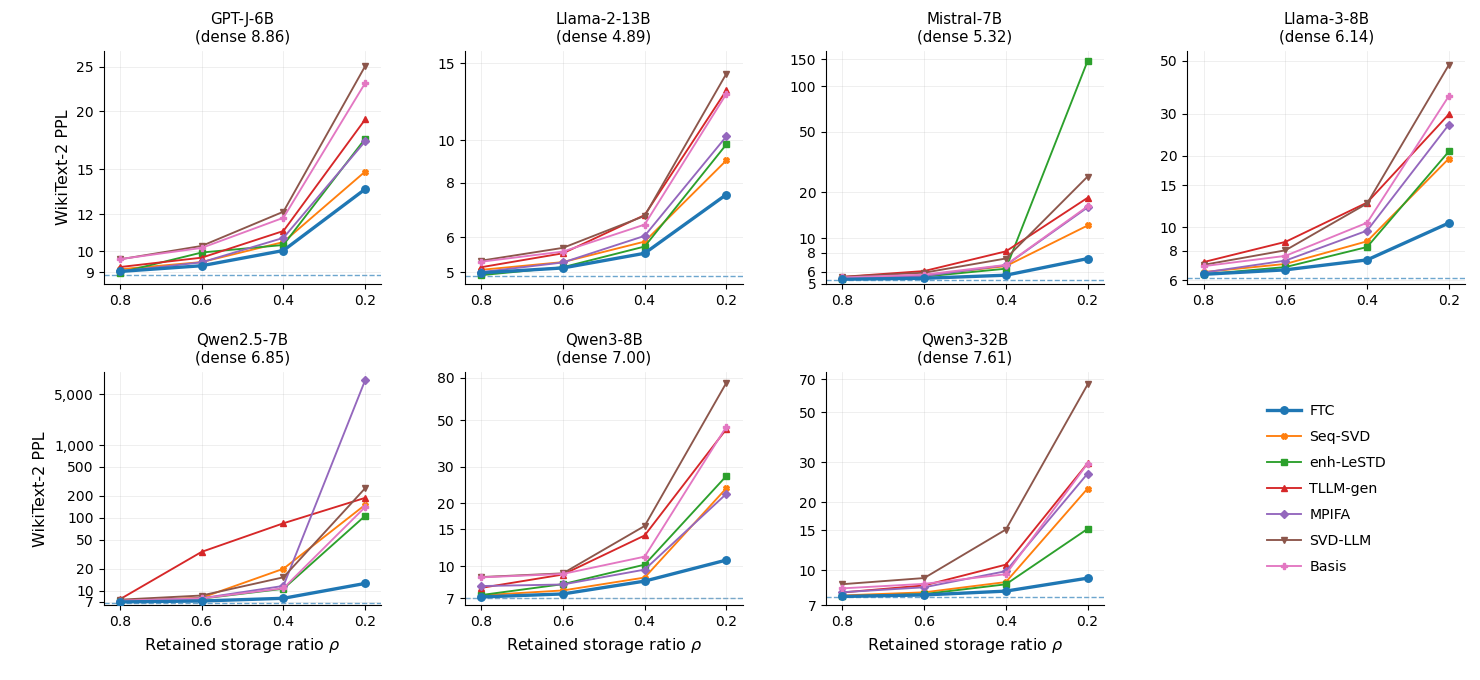}
\caption{WikiText-2 perplexity versus retained attention-projection
storage ratio across seven models. Dashed lines denote the corresponding
dense-model perplexity.}
\label{fig:sweep}
\end{figure}

The results reveal three consistent trends across model families and
compression levels.

\textbf{(1) Aggressive compression magnifies the performance gap.} The gap widens sharply as compression becomes more
aggressive. At $\rrho=0.2$, with WikiText-2 calibration, \ftc{} achieves the best
perplexity on all five modern GQA models and remains within
$1.21$--$1.83\times$ dense WikiText-2 perplexity.
In contrast, SVD-LLM reaches $4.8$--$37.3\times$ dense perplexity. At $\rrho=0.2$, \ftc{} reduces perplexity by $39$--$88\%$ relative to the
best competing baseline on each GQA model.

\textbf{(2) Joint tensorization provides substantial gains beyond the metric
and sequential calibration.}  Seq-SVD serves as a controlled per-matrix counterpart to
\ftc{}, sharing the same activation-induced metric and sequential calibration while factorizing each projection independently. At $\rrho=0.2$, \ftc{} achieves
$39$--$92\%$ lower perplexity across the five GQA models, highlighting the
substantial contribution of joint tensorization beyond the activation-induced metric and calibration alone.

\textbf{(3) The advantage persists at 32B scale.} 
On Qwen3-32B, FTC reaches $9.22$ perplexity at $\rrho=0.2$, or
$1.21\times$ the dense perplexity, compared with $15.23$ for the
strongest competing baseline ($39.5\%$ lower). On Qwen3-8B, the
corresponding reduction is $51.9\%$ ($10.70$ vs.\ $22.23$). Thus,
FTC retains the lowest perplexity at both scales, although the relative
margin is smaller on the 32B model.

\subsection{Downstream tasks}
\label{sec:downstream}

We next examine whether the perplexity improvements translate to downstream
performance. Table~\ref{tab:downstream} reports avg6 and 5-shot GSM8K accuracy
at representative storage ratios, while Figure~\ref{fig:downstream} shows
performance retention across the full compression sweep.

\begin{table}[htbp]
\centering\small
\caption{Downstream performance on modern GQA models at $\rrho\in\{0.6,0.4,0.2\}$. avg6 is the mean normalized accuracy over six
multiple-choice tasks; GSM8K is 5-shot exact match. Dense results are shown in the row labels; full per-task results are in Appendix~\ref{app:full-tasks}.}
\label{tab:downstream}
\setlength{\tabcolsep}{2pt}
\resizebox{\linewidth}{!}{%
\begin{tabular}{llcc|cccccccccccc}
\toprule
& & \multicolumn{2}{c|}{\ftc{}} & \multicolumn{2}{c}{Seq-SVD} & \multicolumn{2}{c}{enh-LeSTD}
& \multicolumn{2}{c}{TLLM-gen} & \multicolumn{2}{c}{MPIFA}
& \multicolumn{2}{c}{SVD-LLM} & \multicolumn{2}{c}{Basis} \\
Model & $\rrho$ & a6 & G8K & a6 & G8K & a6 & G8K & a6 & G8K & a6 & G8K & a6 & G8K & a6 & G8K \\
\midrule
\multirow{3}{*}{\shortstack[l]{Mistral-7B\\(0.740 / 0.285)}} & 0.6 & \cellcolor{best}\textbf{0.726} & \cellcolor{best}\textbf{0.24} & 0.684 & 0.17 & 0.707 & 0.15 & 0.654 & 0.10 & 0.695 & 0.20 & 0.667 & 0.17 & 0.688 & 0.20 \\
 & 0.4 & \cellcolor{best}\textbf{0.689} & \cellcolor{best}\textbf{0.10} & 0.569 & 0.01 & 0.625 & 0.01 & 0.501 & 0.01 & 0.552 & 0.02 & 0.505 & 0.02 & 0.549 & 0.03 \\
 & 0.2 & \cellcolor{best}\textbf{0.512} & 0.01 & 0.359 & 0.00 & 0.320 & 0.00 & 0.342 & 0.01 & 0.314 & 0.01 & 0.307 & 0.01 & 0.294 & 0.00 \\
\midrule
\multirow{3}{*}{\shortstack[l]{Llama-3-8B\\(0.735 / 0.435)}} & 0.6 & \cellcolor{best}\textbf{0.716} & \cellcolor{best}\textbf{0.29} & 0.642 & 0.09 & 0.694 & 0.20 & 0.597 & 0.04 & 0.639 & 0.01 & 0.575 & 0.02 & 0.615 & 0.04 \\
 & 0.4 & \cellcolor{best}\textbf{0.668} & \cellcolor{best}\textbf{0.10} & 0.510 & 0.04 & 0.601 & 0.04 & 0.447 & 0.03 & 0.448 & 0.02 & 0.410 & 0.04 & 0.474 & 0.03 \\
 & 0.2 & \cellcolor{best}\textbf{0.488} & 0.02 & 0.345 & 0.01 & 0.357 & 0.01 & 0.327 & 0.01 & 0.306 & 0.00 & 0.308 & 0.03 & 0.315 & 0.01 \\
\midrule
\multirow{3}{*}{\shortstack[l]{Qwen2.5-7B\\(0.717 / 0.820)}} & 0.6 & \cellcolor{best}\textbf{0.689} & \cellcolor{best}\textbf{0.71} & 0.631 & 0.45 & 0.658 & 0.49 & 0.326 & 0.01 & 0.666 & 0.53 & 0.624 & 0.42 & 0.669 & 0.54 \\
 & 0.4 & \cellcolor{best}\textbf{0.659} & \cellcolor{best}\textbf{0.45} & 0.332 & 0.01 & 0.490 & 0.01 & 0.300 & 0.01 & 0.456 & 0.04 & 0.398 & 0.02 & 0.522 & 0.09 \\
 & 0.2 & \cellcolor{best}\textbf{0.405} & 0.01 & 0.305 & 0.01 & 0.298 & 0.00 & 0.297 & 0.01 & 0.289 & 0.03 & 0.297 & 0.01 & 0.296 & 0.01 \\
\midrule
\multirow{3}{*}{\shortstack[l]{Qwen3-8B\\(0.725 / 0.855)}} & 0.6 & \cellcolor{best}\textbf{0.707} & \cellcolor{best}\textbf{0.73} & 0.668 & 0.46 & 0.644 & 0.54 & 0.586 & 0.23 & 0.670 & 0.44 & 0.601 & 0.20 & 0.661 & 0.20 \\
 & 0.4 & \cellcolor{best}\textbf{0.624} & \cellcolor{best}\textbf{0.10} & 0.561 & 0.09 & 0.529 & 0.07 & 0.432 & 0.01 & 0.522 & 0.08 & 0.425 & 0.01 & 0.561 & 0.04 \\
 & 0.2 & \cellcolor{best}\textbf{0.500} & 0.01 & 0.342 & 0.01 & 0.329 & 0.01 & 0.304 & 0.00 & 0.358 & 0.00 & 0.322 & 0.01 & 0.335 & 0.01 \\
\midrule
\multirow{3}{*}{\shortstack[l]{Qwen3-32B\\(0.743 / 0.885)}} & 0.6 & \cellcolor{best}\textbf{0.733} & \cellcolor{best}\textbf{0.77} & 0.730 & 0.64 & 0.717 & 0.77 & 0.687 & 0.70 & 0.722 & 0.71 & 0.684 & 0.67 & 0.701 & 0.60 \\
 & 0.4 & \cellcolor{best}\textbf{0.733} & \cellcolor{best}\textbf{0.75} & 0.679 & 0.55 & 0.682 & 0.51 & 0.618 & 0.21 & 0.629 & 0.41 & 0.525 & 0.04 & 0.657 & 0.44 \\
 & 0.2 & \cellcolor{best}\textbf{0.671} & 0.28 & 0.375 & 0.00 & 0.460 & 0.01 & 0.355 & 0.00 & 0.372 & 0.01 & 0.335 & 0.00 & 0.369 & 0.01 \\ 
\bottomrule
\end{tabular}}
\end{table}

\begin{figure}[htbp]
\centering
\includegraphics[width=\linewidth]{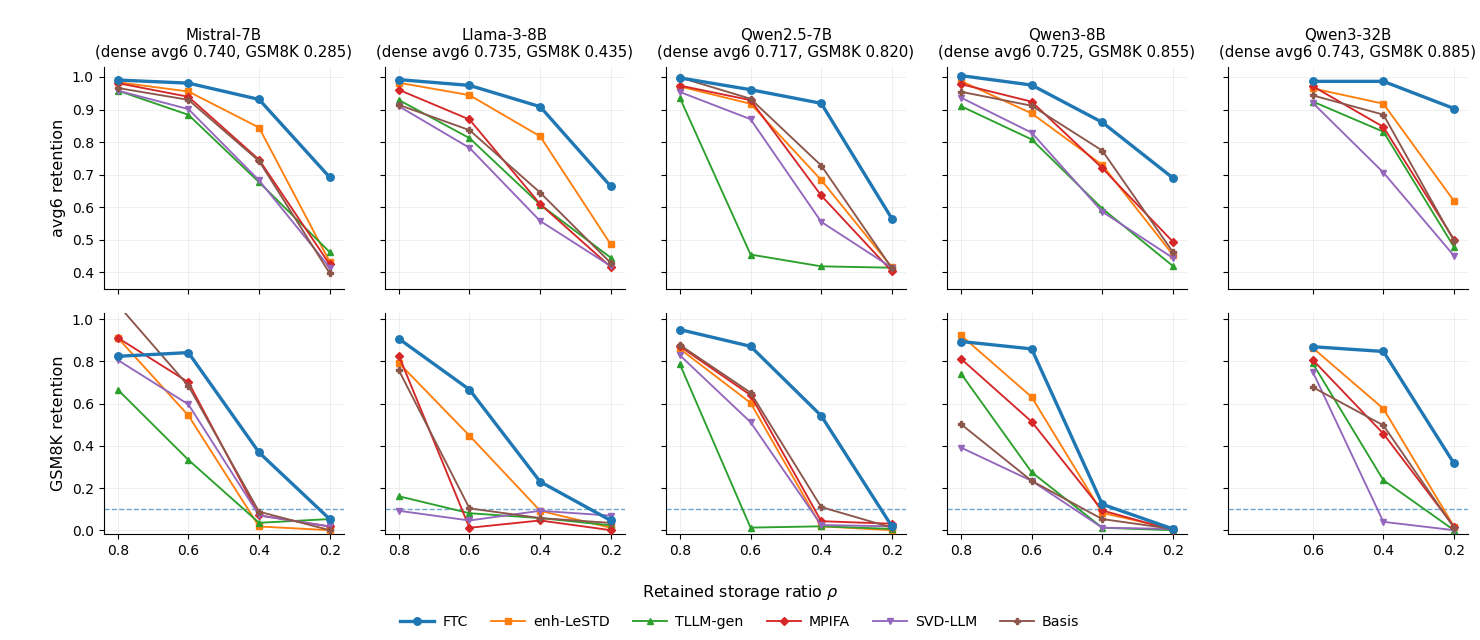}
\caption{Downstream retention versus retained storage ratio on five GQA
models. Top: avg6 relative to dense; bottom: 5-shot GSM8K relative to
dense. The dashed line marks 10\% GSM8K retention.}
\label{fig:downstream}
\end{figure}

The downstream results reveal two main patterns across tasks, model scales, and compression levels.

\textbf{(1) Generative reasoning is more sensitive to compression than
multiple-choice performance.} At $\rrho=0.4$, \ftc{} retains
$86$--$93\%$ of the dense avg6 on the four 7--8B models, whereas its
GSM8K retention ranges from only $12$--$55\%$. This gap is substantially
smaller at larger scale: Qwen3-32B retains $98.7\%$ of its dense avg6 and
$84.7\%$ of its dense GSM8K accuracy at the same compression ratio. These
results indicate that reasoning-intensive generation degrades earlier under
compression, while increased model scale provides considerably greater
robustness.

\textbf{(2) Performance differences become more pronounced under aggressive compression.} At $\rrho=0.2$, \ftc{} exceeds the best
published baseline by $10.7$--$17.0$ percentage points in avg6 across the
four 7--8B models, and by $21.1$ points on Qwen3-32B. Thus, even when
GSM8K has collapsed to near-zero accuracy on smaller models, the
multiple-choice aggregate continues to reveal substantial differences in
the quality retained by different compression schemes.

\subsection{Attribution and ablations}
\label{sec:ablation}
We next isolate the contribution of the main components of FTC through
controlled ablations under the same realized compression budget.
Table~\ref{tab:ablation} reports the resulting WikiText-2 perplexity after
removing or replacing one component at a time.

\begin{table}[H]
\centering\scriptsize
\caption{Component ablations on WikiText-2. ``full'' denotes the final
fixed head-concat recipe; each column removes or replaces one component at
the same realized storage ratio and reports perplexity with the relative
change in parentheses. Relative changes are computed against the matching
full-recipe result in the same row.}
\label{tab:ablation}
\setlength{\tabcolsep}{3pt}
\renewcommand{\arraystretch}{1.20}
\resizebox{\linewidth}{!}{%
\begin{tabular}{llccccccccc}
\toprule
Model & $\rrho$ & full & w/o whitening & w/o slice-wise norm & w/o seq.\ calibration & w/o anchoring & K/V repeat-expand & dense core & no joint tensor & w/o $W_O$ target \\ 
\midrule
Mistral-7B & 0.6 & 5.440 & \cellcolor{worse!12}5.56 (+2.2\%) & \cellcolor{worse!11}5.50 (+1.1\%) & 5.46 (+0.4\%) & 5.44 (+0.0\%) & \cellcolor{worse!11}5.53 (+1.6\%) & \cellcolor{worse!30}6.37 (+17.2\%) & \cellcolor{worse!15}5.69 (+4.7\%) & 5.46 (+0.4\%) \\
Mistral-7B & 0.2 & 7.334 & \cellcolor{worse!70}13.19 (+79.8\%) & \cellcolor{worse!42}9.33 (+27.2\%) & \cellcolor{worse!70}13.24 (+80.6\%) & \cellcolor{worse!11}7.45 (+1.5\%) & \cellcolor{worse!57}10.41 (+41.9\%) & \cellcolor{worse!15}7.65 (+4.3\%) & \cellcolor{worse!70}12.10 (+65.0\%) & \cellcolor{worse!56}10.15 (+38.4\%) \\
Llama-3-8B & 0.6 & 6.640 & \cellcolor{worse!11}6.73 (+1.4\%) & \cellcolor{worse!16}7.01 (+5.5\%) & \cellcolor{worse!15}6.94 (+4.5\%) & \cellcolor{better!11}6.60 (-0.6\%) & \cellcolor{worse!14}6.87 (+3.4\%) & \cellcolor{worse!42}8.45 (+27.3\%) & \cellcolor{worse!17}7.03 (+5.9\%) & \cellcolor{worse!15}6.94 (+4.5\%) \\
Llama-3-8B & 0.2 & 10.440 & \cellcolor{worse!70}20.56 (+96.9\%) & \cellcolor{worse!70}19.77 (+89.4\%) & \cellcolor{worse!70}33.65 (+222.4\%) & \cellcolor{worse!14}10.87 (+4.1\%) & \cellcolor{worse!70}17.18 (+64.6\%) & \cellcolor{worse!18}11.19 (+7.2\%) & \cellcolor{worse!70}19.46 (+86.4\%) & \cellcolor{worse!70}29.50 (+182.6\%) \\
Qwen2.5-7B & 0.6 & 7.197 & \cellcolor{worse!19}7.77 (+7.9\%) & 7.17 (-0.3\%) & \cellcolor{worse!10}7.24 (+0.5\%) & \cellcolor{worse!10}7.24 (+0.5\%) & \cellcolor{worse!10}7.24 (+0.6\%) & \cellcolor{worse!70}12.71 (+76.6\%) & \cellcolor{worse!23}7.98 (+10.9\%) & 7.19 (-0.1\%) \\
Qwen2.5-7B & 0.2 & 12.548 & \cellcolor{worse!70}221.15 (+1662.4\%) & \cellcolor{worse!70}74.34 (+492.4\%) & \cellcolor{worse!70}21.36 (+70.2\%) & \cellcolor{worse!27}14.36 (+14.5\%) & \cellcolor{worse!13}12.88 (+2.6\%) & \cellcolor{worse!70}26.89 (+114.3\%) & \cellcolor{worse!70}150.93 (+1102.8\%) & \cellcolor{worse!13}12.88 (+2.6\%) \\
Qwen3-8B & 0.6 & 7.362 & \cellcolor{worse!23}8.22 (+11.7\%) & 7.35 (-0.2\%) & \cellcolor{worse!12}7.52 (+2.2\%) & \cellcolor{worse!11}7.42 (+0.8\%) & \cellcolor{worse!12}7.52 (+2.2\%) & \cellcolor{worse!41}9.28 (+26.1\%) & \cellcolor{worse!14}7.65 (+3.9\%) & \cellcolor{worse!12}7.54 (+2.5\%) \\
Qwen3-8B & 0.2 & 10.700 & \cellcolor{worse!70}32.28 (+201.7\%) & \cellcolor{worse!24}11.96 (+11.8\%) & \cellcolor{worse!70}59.68 (+457.7\%) & \cellcolor{worse!17}11.35 (+6.1\%) & \cellcolor{worse!52}14.53 (+35.8\%) & \cellcolor{worse!20}11.63 (+8.7\%) & \cellcolor{worse!70}23.66 (+121.2\%) & \cellcolor{worse!60}15.24 (+42.4\%) \\ 
\bottomrule
\end{tabular}}
\end{table}
\vspace{-0.5cm}
The ablation results reveal three main patterns in how the individual components contribute across compression levels and model architectures.

\textbf{(1) Metric whitening, sequential calibration, and joint
tensorization become increasingly important as the compression budget
tightens.} At $\rrho=0.6$, removing whitening,
sequential calibration, or joint tensorization increases perplexity by at
most $11.7\%$. At $\rrho=0.2$, the corresponding penalties rise to
$79.8$--$1662.4\%$, $70.2$--$457.7\%$, and $65.0$--$1102.8\%$,
respectively. These results show that activation-induced metric, adaptation to the evolving student stream, and joint Q/K/V modeling become increasingly
important as the compression budget tightens.

\textbf{(2) The contributions of the remaining components vary more substantially across models.}
At $\rrho=0.2$, removing anchoring changes perplexity by only
$1.5$--$14.5\%$, whereas removing slice-wise normalization costs
$11.8$--$492.4\%$. Replacing the sparse core with its matched-budget dense
counterpart incurs only $4.3$--$8.7\%$ on three models but $114.3\%$ on
Qwen2.5-7B. The $W_O$ teacher target is beneficial on all four models,
with its removal increasing perplexity by $2.6$--$182.6\%$.

\textbf{(3) Fixed head-concat is consistently preferred to K/V
repeat-expansion.}
At $\rrho=0.2$, replacing head-concat with K/V repeat-expansion increases
perplexity by $41.9\%$ on Mistral-7B, $64.6\%$ on Llama-3-8B,
$2.6\%$ on Qwen2.5-7B, and $35.8\%$ on Qwen3-8B. The magnitude varies
substantially across architectures, but the native head-concat construction
is preferred on all four GQA models. Additional teacher-map and
hyperparameter studies are reported in Appendix~\ref{app:hyper}.

\section{Conclusion}

We proposed \ftc{}, a structured attention-compression framework combining
an activation-induced metric, sequential calibration, joint Q/K/V
tensorization, and budget-coupled rank--sparsity allocation. Across MHA
and GQA models from 6B to 32B, \ftc{} achieves consistently lower
perplexity and stronger downstream retention than the compared methods,
without fine-tuning or gradient-based recovery. Ablations further show
that metric whitening, sequential calibration, and joint tensorization
become increasingly important as the compression budget tightens.

\bibliography{iclr2027_conference}
\bibliographystyle{iclr2027_conference}

\appendix

\newpage
\section*{AI Use Statement}

In preparing this work, we used generative AI tools to obtain feedback on
the research methodology and experimental design, review mathematical
derivations for potential errors and inconsistencies, and assist with the
interpretation of experimental results. We also used these tools to draft
and revise portions of the manuscript, improve its readability and
organization, identify relevant literature, and prepare a draft workflow
diagram.

We have reviewed all AI-assisted material retained in this work.
Specifically, we independently checked the mathematical arguments,
verified references against the original sources, and cross-checked
numerical statements against our experimental records. We revised
AI-assisted text and figures to ensure that they accurately describe
the methods and results. We take full responsibility for the final
content of this work, including all text, claims, and artifacts produced
with the assistance of generative AI.

\section{Experimental details and full results}

\subsection{Reproduction protocol and accounting}
\label{app:protocol}

\paragraph{LeSTD reproduction.}
We reproduce LeSTD from its published description and first verify the
implementation against its reported GPT-J results. Native LeSTD is not
directly applicable to GQA because its per-head factorization assumes
one-to-one Q/K/V head structure. We therefore use enh-LeSTD, a
GQA-compatible port that adopts the same native head-concat tensorization
as \ftc{} while retaining LeSTD's data-free objective. This isolates the
compression objective from the tensorization required for GQA support.

\paragraph{Accounting.}
\label{app:accounting}
We define $\rrho$ as the retained storage ratio of the compressed attention projections. For Q/K/V, the accounting includes the shared Tucker factors, packed nonzero core values, and the bit-packed support mask; $W_{O,l}$ is compressed and counted separately. Thus, sparse-core
metadata is included in $\rrho$. Let $s_{\mathrm{attn}}$ denote the fraction of dense model storage occupied by the compressed attention projections. Since all dense weights
use the same precision, this equals their parameter fraction. Leaving the
remaining model unchanged, the corresponding whole-model storage reduction
is $(1-\rho)s_{\mathrm{attn}}.$
Table~\ref{tab:accounting} reports this conversion. 

\begin{table}[htbp]
\centering
\caption{Storage accounting at each attention-projection storage budget
$\rrho$. Realized storage is determined by the selected layerwise Tucker
ranks $(R_{1,l},R_{2,l})$, sparse-core cardinality $K_l$,
output-projection rank $r_{O,l}$, and bit-packed support mask.}
\label{tab:accounting}
\scriptsize
\setlength{\tabcolsep}{3.5pt}
\begin{tabular}{@{}lccccc@{}}
\toprule
& & \multicolumn{4}{c}{Whole-model storage removed (\%)} \\
\cmidrule(lr){3-6}
Model
& Attn.\ share & $\rrho=.8$ & $\rrho=.6$ & $\rrho=.4$ & $\rrho=.2$ \\
\midrule
GPT-J-6B     & 31.1\% & 6.2 & 12.4 & 18.7 & 24.9 \\
Llama-2-13B  & 32.2\% & 6.4 & 12.9 & 19.3 & 25.8 \\
Mistral-7B    & 18.5\% & 3.7 & 7.4  & 11.1 & 14.8 \\
Llama-3-8B    & 16.7\% & 3.3 & 6.7  & 10.0 & 13.4 \\
Qwen2.5-7B    & 10.8\% & 2.2 & 4.3  & 6.5  & 8.6 \\
Qwen3-8B      & 18.4\% & 3.7 & 7.4  & 11.0 & 14.7 \\

\bottomrule
\end{tabular}
\end{table}

\paragraph{Calibration and evaluation.}
\ftc{} uses 256 calibration sequences of 2{,}048 tokens (WikiText-2 train,
seed 42) on every model; the data-dependent baselines use the 128 sequences
of their published configurations, and Table~\ref{tab:calibsize} quantifies the effect of this difference. Perplexity on
WikiText-2 test with stride-2048 windows; PTB and C4 in
App.~\ref{app:full-ppl}. Tasks: 1{,}000 examples per task (983 for MathQA, 200 for GSM8K),
acc\_norm for multiple choice, exact-match accuracy for 5-shot GSM8K
(batch 8, greedy).

\paragraph{Large-model specifics.}
Qwen3-32B does not fit one 40GB card in bf16, so compression runs
layer-at-a-time with the weights on CPU (only the embedding, the shared
rotary table and one decoder layer are resident) and task evaluation shards
the model across two cards with CPU spill. Nothing else changes: same recipe,
same 256 calibration windows, same accounting.

\paragraph{Compression cost.}
Table~\ref{tab:cost} reports the wall-clock time of the complete \ftc{}
pipeline. On a single A100-40GB, the 7--8B models require 16--39 minutes per
compression setting, GPT-J-6B requires 21--38 minutes, and Llama-2-13B
requires 39--166 minutes. Qwen3-32B requires 96--170 minutes with CPU
offloading and layer-wise compression. No gradient-based optimization is
used; the dominant costs are Gram eigendecomposition and HOOI/refit
iterations.

\begin{table}[htbp]
\centering\scriptsize
\caption{\ftc{} compression wall-clock time (minutes) per model--ratio
configuration. Columns correspond to $\rrho=0.8/0.6/0.4/0.2$.
All models use the fixed head-concat recipe on a single A100-40GB GPU;
Qwen3-32B additionally uses CPU offloading and is compressed one decoder
layer at a time.}
\label{tab:cost}
\begin{tabular}{llcccc}
\toprule
Model & mode & 0.8 & 0.6 & 0.4 & 0.2 \\
\midrule
GPT-J-6B & 1 card & 38 & 34 & 26 & 21 \\
Llama-2-13B & 1 card & 166 & 93 & 75 & 39 \\
Mistral-7B & 1 card & 33 & 30 & 26 & 26 \\
Llama-3-8B & 1 card & 34 & 29 & 27 & 25 \\
Qwen2.5-7B & 1 card & 21 & 20 & 16 & 16 \\
Qwen3-8B & 1 card & 39 & 33 & 28 & 27 \\
Qwen3-32B & 1 card & 170 & 137 & 106 & 96 \\ 
\bottomrule
\end{tabular}
\end{table}

\subsection{Full perplexity tables}
\label{app:full-ppl}

\begin{table}[H]
\centering\scriptsize
\caption{Full WikiText-2 perplexity results over all models and
$\rrho\in\{0.8,0.6,0.4,0.2\}$. Best results in each row are highlighted.}
\setlength{\tabcolsep}{4pt}
\resizebox{\linewidth}{!}{%
\begin{tabular}{llc|cccccc}
\toprule
Model & $\rrho$ & \ftc{} & Seq-SVD & enh-LeSTD & TLLM-gen & MPIFA & SVD-LLM & Basis \\
\midrule
GPT-J-6B & 0.8 & 9.04 & 9.14 & \cellcolor{best}\textbf{8.99} & 9.24 & 9.04 & 9.60 & 9.61 \\
 & 0.6 & \cellcolor{best}\textbf{9.30} & 9.47 & 9.93 & 9.68 & 9.46 & 10.26 & 10.16 \\
 & 0.4 & \cellcolor{best}\textbf{10.03} & 10.46 & 10.31 & 11.05 & 10.68 & 12.16 & 11.80 \\
 & 0.2 & \cellcolor{best}\textbf{13.60} & 14.84 & 17.47 & 19.22 & 17.24 & 25.04 & 23.02 \\
\midrule
Llama-2-13B & 0.8 & 4.98 & 5.06 & \cellcolor{best}\textbf{4.92} & 5.13 & 5.00 & 5.31 & 5.28 \\
 & 0.6 & \cellcolor{best}\textbf{5.11} & 5.27 & 5.13 & 5.52 & 5.26 & 5.68 & 5.57 \\
 & 0.4 & \cellcolor{best}\textbf{5.52} & 5.87 & 5.72 & 6.75 & 6.05 & 6.72 & 6.42 \\
 & 0.2 & \cellcolor{best}\textbf{7.52} & 9.00 & 9.79 & 13.01 & 10.21 & 14.20 & 12.77 \\
\midrule
Mistral-7B & 0.8 & \cellcolor{best}\textbf{5.36} & 5.43 & 5.40 & 5.56 & 5.39 & 5.56 & 5.51 \\
 & 0.6 & \cellcolor{best}\textbf{5.44} & 5.69 & 5.58 & 6.07 & 5.63 & 5.92 & 5.75 \\
 & 0.4 & \cellcolor{best}\textbf{5.70} & 6.54 & 6.29 & 8.18 & 6.63 & 7.34 & 6.66 \\
 & 0.2 & \cellcolor{best}\textbf{7.33} & 12.10 & 146.31 & 18.41 & 15.92 & 25.40 & 16.18 \\
\midrule
Llama-3-8B & 0.8 & \cellcolor{best}\textbf{6.36} & 6.50 & 6.36 & 7.16 & 6.47 & 6.97 & 6.88 \\
 & 0.6 & \cellcolor{best}\textbf{6.64} & 7.03 & 6.82 & 8.70 & 7.26 & 8.02 & 7.60 \\
 & 0.4 & \cellcolor{best}\textbf{7.31} & 8.75 & 8.28 & 12.72 & 9.70 & 12.61 & 10.47 \\
 & 0.2 & \cellcolor{best}\textbf{10.44} & 19.46 & 20.84 & 29.85 & 27.02 & 48.09 & 35.60 \\
\midrule
Qwen2.5-7B & 0.8 & \cellcolor{best}\textbf{6.99} & 7.21 & 7.17 & 7.61 & 7.16 & 7.50 & 7.28 \\
 & 0.6 & \cellcolor{best}\textbf{7.20} & 7.98 & 7.76 & 34.10 & 7.84 & 8.56 & 7.82 \\
 & 0.4 & \cellcolor{best}\textbf{7.84} & 19.96 & 10.60 & 84.22 & 11.60 & 15.23 & 10.77 \\
 & 0.2 & \cellcolor{best}\textbf{12.55} & 150.93 & 106.59 & 188.22 & 7791 & 255.23 & 140.65 \\
\midrule
Qwen3-8B & 0.8 & \cellcolor{best}\textbf{7.12} & 7.25 & 7.27 & 7.87 & 8.03 & 8.87 & 8.84 \\
 & 0.6 & \cellcolor{best}\textbf{7.36} & 7.65 & 8.20 & 9.12 & 8.17 & 9.24 & 9.18 \\
 & 0.4 & \cellcolor{best}\textbf{8.48} & 8.82 & 10.17 & 14.04 & 9.63 & 15.59 & 11.11 \\
 & 0.2 & \cellcolor{best}\textbf{10.70} & 23.66 & 27.08 & 45.45 & 22.23 & 75.77 & 46.68 \\
\midrule
Qwen3-32B & 0.8 & \cellcolor{best}\textbf{7.65} & 7.73 & 7.66 & 7.97 & 7.99 & 8.67 & 8.31 \\
 & 0.6 & \cellcolor{best}\textbf{7.77} & 7.98 & 7.84 & 8.56 & 8.40 & 9.23 & 8.70 \\
 & 0.4 & \cellcolor{best}\textbf{8.08} & 8.86 & 8.66 & 10.58 & 9.90 & 15.13 & 9.60 \\
 & 0.2 & \cellcolor{best}\textbf{9.22} & 22.97 & 15.23 & 29.71 & 26.75 & 66.29 & 29.63 \\ 
\bottomrule
\end{tabular}}
\end{table}

\begin{table}[H]
\centering\scriptsize
\caption{PTB (left block) and C4 (right block) perplexity, method order as
in Table~\ref{tab:main} (F = \ftc{}, qS = Seq-SVD, eL = enh-LeSTD,
TG = TLLM-gen, MP = MPIFA, SL = SVD-LLM, BS = Basis).
Mistral's high PTB values reflect its tokenizer/domain mismatch and affect
all methods equally.}
\setlength{\tabcolsep}{2pt}
\resizebox{\linewidth}{!}{%
\begin{tabular}{llccccccc|ccccccc}
\toprule
& & \multicolumn{7}{c|}{PTB} & \multicolumn{7}{c}{C4} \\
Model & $\rrho$ & F & qS & eL & TG & MP & SL & BS & F & qS & eL & TG & MP & SL & BS \\
\midrule
GPT-J-6B & 0.8 & 16.33 & 16.80 & \textbf{16.28} & 17.03 & 16.30 & 17.45 & 17.42 & 13.62 & 14.12 & \textbf{13.29} & 14.06 & 13.62 & 14.84 & 14.83 \\
 & 0.6 & \textbf{17.20} & 18.39 & 17.85 & 18.43 & 17.46 & 19.74 & 19.44 & 14.33 & 15.24 & \textbf{14.18} & 15.09 & 14.82 & 16.45 & 16.27 \\
 & 0.4 & \textbf{20.34} & 25.08 & 21.93 & 25.25 & 24.20 & 30.19 & 28.64 & \textbf{16.61} & 18.80 & 16.94 & 18.75 & 18.69 & 21.89 & 21.26 \\
 & 0.2 & \textbf{50.52} & 67.11 & 82.87 & 101.12 & 88.92 & 162.55 & 115.96 & \textbf{29.12} & 34.64 & 40.79 & 42.92 & 42.27 & 65.53 & 59.42 \\
\midrule
Llama-2-13B & 0.8 & \textbf{52.27} & 61.58 & 53.14 & 54.14 & 56.42 & 61.75 & 60.82 & 7.05 & 7.60 & \textbf{7.01} & 7.45 & 7.04 & 8.22 & 8.14 \\
 & 0.6 & \textbf{56.73} & 75.51 & 58.10 & 62.43 & 65.56 & 79.16 & 78.79 & 7.74 & 8.79 & \textbf{7.73} & 8.73 & 8.47 & 9.93 & 9.64 \\
 & 0.4 & 118.70 & 118.43 & \textbf{76.41} & 100.17 & 136.49 & 163.98 & 140.54 & \textbf{9.39} & 11.19 & 9.84 & 12.20 & 11.76 & 14.47 & 13.39 \\
 & 0.2 & 226.93 & 244.05 & \textbf{207.72} & 249.05 & 286.09 & 487.70 & 478.87 & \textbf{16.53} & 21.84 & 22.85 & 30.64 & 28.65 & 43.23 & 38.24 \\
\midrule
Mistral-7B & 0.8 & 39.73 & 38.55 & 38.46 & 42.97 & \textbf{38.26} & 39.49 & 40.60 & \textbf{8.67} & 9.07 & 8.75 & 9.20 & 8.79 & 9.30 & 9.12 \\
 & 0.6 & 43.16 & 43.52 & 42.85 & 55.49 & \textbf{41.47} & 48.17 & 46.79 & \textbf{8.91} & 10.20 & 9.29 & 10.64 & 9.78 & 10.76 & 10.11 \\
 & 0.4 & \textbf{50.17} & 65.37 & 63.35 & 120.65 & 71.57 & 97.88 & 77.40 & \textbf{9.71} & 13.41 & 11.47 & 16.46 & 14.00 & 17.23 & 14.10 \\
 & 0.2 & \textbf{114.66} & 233.04 & 1202 & 487.88 & 299.66 & 479.05 & 350.76 & \textbf{15.91} & 34.86 & 562.15 & 49.10 & 51.15 & 82.38 & 54.43 \\
\midrule
Llama-3-8B & 0.8 & \textbf{11.70} & 12.47 & 11.72 & 13.51 & 12.03 & 13.47 & 13.05 & \textbf{10.09} & 14.06 & 10.12 & 12.58 & 10.57 & 20.88 & 11.82 \\
 & 0.6 & \textbf{12.44} & 14.54 & 13.13 & 17.76 & 14.70 & 17.51 & 15.92 & \textbf{11.01} & 20.44 & 11.71 & 16.67 & 20.53 & 38.23 & 15.03 \\
 & 0.4 & \textbf{15.43} & 23.04 & 19.24 & 36.09 & 29.62 & 47.80 & 34.60 & \textbf{13.16} & 23.36 & 16.62 & 29.47 & 31.70 & 54.62 & 31.45 \\
 & 0.2 & \textbf{50.71} & 107.02 & 111.94 & 253.19 & 128.53 & 264.41 & 217.80 & \textbf{29.18} & 90.63 & 69.35 & 103.12 & 132.05 & 253.66 & 180.43 \\
\midrule
Qwen2.5-7B & 0.8 & \textbf{13.28} & 14.51 & 13.96 & 15.24 & 13.62 & 14.89 & 14.10 & \textbf{11.97} & 15.49 & 12.48 & 13.38 & 13.78 & 15.62 & 12.69 \\
 & 0.6 & \textbf{14.26} & 18.97 & 17.19 & 129.63 & 16.93 & 19.39 & 16.79 & \textbf{12.50} & 21.10 & 13.81 & 95.39 & 19.00 & 22.11 & 16.06 \\
 & 0.4 & \textbf{18.20} & 110.39 & 39.29 & 586.17 & 43.70 & 76.62 & 36.03 & \textbf{15.42} & 138.98 & 27.67 & 306.41 & 43.86 & 70.26 & 33.25 \\
 & 0.2 & \textbf{89.42} & 2112 & 2477 & 5202 & 55759 & 3251 & 1760 & \textbf{73.37} & 1106 & 732.57 & 1263 & 26390 & 1406 & 786.30 \\
\midrule
Qwen3-8B & 0.8 & \textbf{12.94} & 13.77 & 13.47 & 15.65 & 14.90 & 16.91 & 16.41 & \textbf{12.09} & 12.77 & 12.48 & 13.92 & 13.85 & 15.38 & 15.13 \\
 & 0.6 & \textbf{13.58} & 15.51 & 17.08 & 21.98 & 16.89 & 19.69 & 18.26 & \textbf{12.72} & 14.21 & 14.76 & 17.59 & 15.18 & 17.61 & 16.70 \\
 & 0.4 & \textbf{16.41} & 22.19 & 31.72 & 68.02 & 26.57 & 48.73 & 27.38 & \textbf{14.78} & 19.16 & 22.21 & 35.16 & 22.44 & 38.76 & 25.39 \\
 & 0.2 & \textbf{33.02} & 106.03 & 264.11 & 636.55 & 103.14 & 806.93 & 204.73 & \textbf{23.73} & 69.32 & 85.84 & 145.38 & 64.87 & 249.44 & 133.10 \\
\midrule
Qwen3-32B & 0.8 & \textbf{13.49} & 14.10 & 13.54 & 14.27 & 14.15 & 16.53 & 15.37 & \textbf{12.47} & 12.98 & 12.51 & 13.06 & 13.02 & 14.45 & 13.63 \\
 & 0.6 & \textbf{13.99} & 15.53 & 14.27 & 16.12 & 16.76 & 20.48 & 17.70 & \textbf{12.79} & 14.25 & 12.99 & 14.36 & 15.05 & 17.70 & 15.30 \\
 & 0.4 & \textbf{15.14} & 24.28 & 18.27 & 26.99 & 39.58 & 71.73 & 24.43 & \textbf{13.68} & 18.74 & 15.18 & 19.35 & 23.69 & 37.84 & 19.81 \\
 & 0.2 & \textbf{24.38} & 220.91 & 132.24 & 455.35 & 299.63 & 600.37 & 291.75 & \textbf{19.26} & 88.34 & 43.53 & 113.39 & 113.11 & 213.61 & 90.84 \\ 
\bottomrule
\end{tabular}}
\end{table}

\subsection{Full downstream grid}
\label{app:full-tasks}

Table~\ref{tab:full-tasks} lists every task row measured: per-task acc\_norm
for the six base tasks, MathQA, GSM8K, and avg6, for every (model, method,
ratio) cell on disk --- the modern-model grid (4 models $\times$ 4--6
methods $\times$ 4 ratios), the MHA models, and the Qwen3.5 hybrid rows.
``--'' marks configurations that are not applicable (diverged native LeSTD
at $\rrho{=}0.2$).

\begin{table}[H]
\centering\small
\caption{Complete downstream grid, MHA models (acc\_norm; GSM8K exact
match).}
\label{tab:full-tasks-mha}
\setlength{\tabcolsep}{2pt}
\begin{tabular}{lllccccccccc}
\toprule
Model & Method & $\rrho$ & PIQA & ARC-e & ARC-c & HSwag & WinoG & LAMB & MathQA & GSM8K & avg6 \\
\midrule
GPT-J-6B & dense & 1.0 & 0.756 & 0.616 & 0.367 & 0.654 & 0.654 & 0.678 & 0.281 & 0.015 & 0.621 \\
 & \ftc{} & 0.8 & 0.747 & 0.607 & 0.358 & 0.642 & 0.659 & 0.665 & 0.269 & 0.010 & 0.613 \\
 & \ftc{} & 0.6 & 0.742 & 0.589 & 0.333 & 0.626 & 0.640 & 0.658 & 0.264 & 0.005 & 0.598 \\
 & \ftc{} & 0.4 & 0.735 & 0.577 & 0.317 & 0.567 & 0.618 & 0.573 & 0.263 & 0.025 & 0.564 \\
 & \ftc{} & 0.2 & 0.684 & 0.464 & 0.270 & 0.438 & 0.575 & 0.255 & 0.258 & 0.000 & 0.448 \\
 & enh-LeSTD & 0.8 & 0.758 & 0.620 & 0.362 & 0.649 & 0.651 & 0.682 & 0.270 & 0.010 & 0.620 \\
 & enh-LeSTD & 0.6 & 0.751 & 0.596 & 0.337 & 0.616 & 0.623 & 0.686 & 0.269 & 0.030 & 0.602 \\
 & enh-LeSTD & 0.4 & 0.728 & 0.544 & 0.304 & 0.561 & 0.608 & 0.499 & 0.248 & 0.010 & 0.541 \\
 & enh-LeSTD & 0.2 & 0.669 & 0.448 & 0.272 & 0.405 & 0.570 & 0.149 & 0.237 & 0.000 & 0.419 \\
 & TLLM-gen & 0.8 & 0.744 & 0.577 & 0.348 & 0.626 & 0.637 & 0.629 & 0.272 & 0.020 & 0.594 \\
 & TLLM-gen & 0.6 & 0.730 & 0.565 & 0.324 & 0.592 & 0.623 & 0.621 & 0.271 & 0.020 & 0.576 \\
 & TLLM-gen & 0.2 & 0.660 & 0.412 & 0.254 & 0.387 & 0.556 & 0.146 & 0.221 & 0.000 & 0.403 \\
 & MPIFA & 0.8 & 0.754 & 0.601 & 0.352 & 0.631 & 0.657 & 0.684 & 0.277 & 0.045 & 0.613 \\
 & MPIFA & 0.6 & 0.747 & 0.584 & 0.338 & 0.594 & 0.641 & 0.689 & 0.272 & 0.010 & 0.599 \\
 & MPIFA & 0.4 & 0.724 & 0.556 & 0.299 & 0.533 & 0.619 & 0.576 & 0.273 & 0.010 & 0.551 \\
 & MPIFA & 0.2 & 0.615 & 0.427 & 0.270 & 0.358 & 0.586 & 0.239 & 0.249 & 0.010 & 0.416 \\
 & SVD-LLM & 0.8 & 0.739 & 0.592 & 0.336 & 0.600 & 0.667 & 0.696 & 0.274 & 0.025 & 0.605 \\
 & SVD-LLM & 0.6 & 0.736 & 0.562 & 0.314 & 0.567 & 0.646 & 0.674 & 0.270 & 0.015 & 0.583 \\
 & SVD-LLM & 0.4 & 0.703 & 0.537 & 0.285 & 0.499 & 0.622 & 0.519 & 0.260 & 0.035 & 0.527 \\
 & SVD-LLM & 0.2 & 0.601 & 0.400 & 0.262 & 0.334 & 0.567 & 0.152 & 0.233 & 0.000 & 0.386 \\
 & Basis Sharing & 0.8 & 0.744 & 0.597 & 0.339 & 0.601 & 0.658 & 0.684 & 0.272 & 0.025 & 0.604 \\
 & Basis Sharing & 0.6 & 0.737 & 0.579 & 0.319 & 0.574 & 0.648 & 0.662 & 0.272 & 0.010 & 0.586 \\
 & Basis Sharing & 0.4 & 0.709 & 0.554 & 0.292 & 0.504 & 0.634 & 0.530 & 0.257 & 0.025 & 0.537 \\
 & Basis Sharing & 0.2 & 0.601 & 0.412 & 0.264 & 0.348 & 0.568 & 0.162 & 0.241 & 0.010 & 0.393 \\
\midrule
Llama-2-13B & dense & 1.0 & 0.793 & 0.777 & 0.482 & 0.789 & 0.720 & 0.777 & 0.329 & 0.205 & 0.723 \\
 & \ftc{} & 0.8 & 0.791 & 0.763 & 0.469 & 0.784 & 0.719 & 0.764 & 0.314 & 0.190 & 0.715 \\
 & \ftc{} & 0.6 & 0.787 & 0.766 & 0.465 & 0.768 & 0.694 & 0.738 & 0.295 & 0.115 & 0.703 \\
 & \ftc{} & 0.4 & 0.768 & 0.726 & 0.416 & 0.715 & 0.690 & 0.647 & 0.273 & 0.030 & 0.660 \\
 & \ftc{} & 0.2 & 0.693 & 0.565 & 0.321 & 0.520 & 0.618 & 0.263 & 0.234 & 0.010 & 0.497 \\
 & enh-LeSTD & 0.8 & 0.793 & 0.767 & 0.475 & 0.782 & 0.723 & 0.750 & 0.315 & 0.195 & 0.715 \\
 & enh-LeSTD & 0.6 & 0.780 & 0.768 & 0.452 & 0.765 & 0.712 & 0.710 & 0.300 & 0.100 & 0.698 \\
 & enh-LeSTD & 0.4 & 0.763 & 0.708 & 0.416 & 0.697 & 0.681 & 0.574 & 0.273 & 0.040 & 0.640 \\
 & enh-LeSTD & 0.2 & 0.646 & 0.453 & 0.270 & 0.438 & 0.573 & 0.145 & 0.226 & 0.005 & 0.421 \\
 & TLLM-gen & 0.8 & 0.781 & 0.752 & 0.454 & 0.758 & 0.708 & 0.717 & 0.303 & 0.170 & 0.695 \\
 & TLLM-gen & 0.6 & 0.762 & 0.727 & 0.412 & 0.708 & 0.691 & 0.643 & 0.279 & 0.075 & 0.657 \\
 & TLLM-gen & 0.4 & 0.697 & 0.586 & 0.337 & 0.603 & 0.636 & 0.430 & 0.247 & 0.015 & 0.548 \\
 & TLLM-gen & 0.2 & 0.606 & 0.393 & 0.249 & 0.380 & 0.556 & 0.049 & 0.222 & 0.000 & 0.372 \\
 & MPIFA & 0.8 & 0.785 & 0.766 & 0.468 & 0.773 & 0.734 & 0.763 & 0.325 & 0.200 & 0.715 \\
 & MPIFA & 0.6 & 0.772 & 0.749 & 0.454 & 0.732 & 0.712 & 0.720 & 0.296 & 0.155 & 0.690 \\
 & MPIFA & 0.4 & 0.730 & 0.627 & 0.366 & 0.640 & 0.680 & 0.532 & 0.249 & 0.010 & 0.596 \\
 & MPIFA & 0.2 & 0.600 & 0.407 & 0.247 & 0.360 & 0.574 & 0.097 & 0.240 & 0.015 & 0.381 \\
 & SVD-LLM & 0.8 & 0.774 & 0.732 & 0.465 & 0.736 & 0.724 & 0.721 & 0.301 & 0.120 & 0.692 \\
 & SVD-LLM & 0.6 & 0.755 & 0.680 & 0.404 & 0.681 & 0.705 & 0.659 & 0.280 & 0.045 & 0.647 \\
 & SVD-LLM & 0.4 & 0.713 & 0.601 & 0.335 & 0.599 & 0.661 & 0.433 & 0.256 & 0.035 & 0.557 \\
 & SVD-LLM & 0.2 & 0.577 & 0.396 & 0.233 & 0.339 & 0.551 & 0.073 & 0.229 & 0.005 & 0.361 \\
 & Basis Sharing & 0.8 & 0.782 & 0.745 & 0.465 & 0.746 & 0.720 & 0.715 & 0.295 & 0.135 & 0.696 \\
 & Basis Sharing & 0.6 & 0.762 & 0.700 & 0.437 & 0.713 & 0.717 & 0.665 & 0.285 & 0.075 & 0.666 \\
 & Basis Sharing & 0.4 & 0.723 & 0.623 & 0.359 & 0.622 & 0.669 & 0.476 & 0.261 & 0.025 & 0.579 \\
 & Basis Sharing & 0.2 & 0.587 & 0.421 & 0.253 & 0.367 & 0.550 & 0.087 & 0.232 & 0.005 & 0.378 \\ 
\bottomrule
\end{tabular}
\end{table}

\begin{table}[H]
\centering\tiny
\caption{Complete downstream grid, GQA and hybrid models.}
\label{tab:full-tasks}
\setlength{\tabcolsep}{2pt}
\begin{tabular}{lllccccccccc}
\toprule
Model & Method & $\rrho$ & PIQA & ARC-e & ARC-c & HSwag & WinoG & LAMB & MathQA & GSM8K & avg6 \\
\midrule
Mistral-7B & dense & 1.0 & 0.814 & 0.807 & 0.530 & 0.796 & 0.740 & 0.753 & 0.367 & 0.285 & 0.740 \\
 & \ftc{} & 0.8 & 0.805 & 0.812 & 0.521 & 0.779 & 0.745 & 0.737 & 0.345 & 0.235 & 0.733 \\
 & \ftc{} & 0.6 & 0.805 & 0.811 & 0.506 & 0.765 & 0.741 & 0.728 & 0.336 & 0.240 & 0.726 \\
 & \ftc{} & 0.4 & 0.795 & 0.781 & 0.474 & 0.726 & 0.701 & 0.655 & 0.303 & 0.105 & 0.689 \\
 & \ftc{} & 0.2 & 0.720 & 0.595 & 0.321 & 0.555 & 0.619 & 0.263 & 0.247 & 0.015 & 0.512 \\
 & Seq-SVD & 0.6 & 0.781 & 0.759 & 0.450 & 0.711 & 0.716 & 0.690 & 0.320 & 0.175 & 0.684 \\
 & Seq-SVD & 0.4 & 0.706 & 0.603 & 0.337 & 0.582 & 0.661 & 0.524 & 0.287 & 0.010 & 0.569 \\
 & Seq-SVD & 0.2 & 0.545 & 0.339 & 0.243 & 0.343 & 0.562 & 0.121 & 0.220 & 0.000 & 0.359 \\
 & enh-LeSTD & 0.8 & 0.805 & 0.807 & 0.509 & 0.774 & 0.736 & 0.738 & 0.345 & 0.260 & 0.728 \\
 & enh-LeSTD & 0.6 & 0.794 & 0.796 & 0.479 & 0.745 & 0.733 & 0.697 & 0.334 & 0.155 & 0.707 \\
 & enh-LeSTD & 0.4 & 0.770 & 0.709 & 0.404 & 0.668 & 0.672 & 0.526 & 0.287 & 0.005 & 0.625 \\
 & enh-LeSTD & 0.2 & 0.524 & 0.300 & 0.252 & 0.295 & 0.535 & 0.011 & 0.213 & 0.000 & 0.320 \\
 & TLLM-gen & 0.8 & 0.798 & 0.790 & 0.484 & 0.748 & 0.725 & 0.711 & 0.331 & 0.190 & 0.709 \\
 & TLLM-gen & 0.6 & 0.770 & 0.738 & 0.426 & 0.699 & 0.695 & 0.594 & 0.314 & 0.095 & 0.654 \\
 & TLLM-gen & 0.4 & 0.694 & 0.540 & 0.320 & 0.552 & 0.632 & 0.270 & 0.242 & 0.010 & 0.501 \\
 & TLLM-gen & 0.2 & 0.566 & 0.337 & 0.259 & 0.335 & 0.537 & 0.016 & 0.209 & 0.015 & 0.342 \\
 & MPIFA & 0.8 & 0.803 & 0.789 & 0.511 & 0.781 & 0.733 & 0.737 & 0.360 & 0.260 & 0.726 \\
 & MPIFA & 0.6 & 0.791 & 0.755 & 0.466 & 0.723 & 0.734 & 0.699 & 0.342 & 0.200 & 0.695 \\
 & MPIFA & 0.4 & 0.710 & 0.581 & 0.320 & 0.539 & 0.653 & 0.508 & 0.276 & 0.020 & 0.552 \\
 & MPIFA & 0.2 & 0.527 & 0.312 & 0.230 & 0.281 & 0.512 & 0.022 & 0.222 & 0.005 & 0.314 \\
 & SVD-LLM & 0.8 & 0.797 & 0.783 & 0.486 & 0.755 & 0.728 & 0.705 & 0.355 & 0.230 & 0.709 \\
 & SVD-LLM & 0.6 & 0.776 & 0.735 & 0.435 & 0.677 & 0.719 & 0.660 & 0.312 & 0.170 & 0.667 \\
 & SVD-LLM & 0.4 & 0.681 & 0.542 & 0.305 & 0.472 & 0.633 & 0.394 & 0.257 & 0.020 & 0.505 \\
 & SVD-LLM & 0.2 & 0.516 & 0.291 & 0.231 & 0.294 & 0.506 & 0.007 & 0.223 & 0.005 & 0.307 \\
 & Basis Sharing & 0.8 & 0.803 & 0.777 & 0.502 & 0.757 & 0.742 & 0.706 & 0.360 & 0.305 & 0.715 \\
 & Basis Sharing & 0.6 & 0.790 & 0.775 & 0.469 & 0.709 & 0.723 & 0.664 & 0.314 & 0.195 & 0.688 \\
 & Basis Sharing & 0.4 & 0.727 & 0.599 & 0.352 & 0.550 & 0.614 & 0.450 & 0.274 & 0.025 & 0.549 \\
 & Basis Sharing & 0.2 & 0.517 & 0.251 & 0.251 & 0.240 & 0.506 & 0.001 & 0.207 & 0.000 & 0.294 \\
\midrule
Llama-3-8B & dense & 1.0 & 0.805 & 0.790 & 0.531 & 0.796 & 0.738 & 0.750 & 0.401 & 0.435 & 0.735 \\
 & \ftc{} & 0.8 & 0.794 & 0.801 & 0.535 & 0.779 & 0.733 & 0.735 & 0.370 & 0.395 & 0.729 \\
 & \ftc{} & 0.6 & 0.788 & 0.786 & 0.501 & 0.766 & 0.740 & 0.717 & 0.326 & 0.290 & 0.716 \\
 & \ftc{} & 0.4 & 0.778 & 0.730 & 0.438 & 0.730 & 0.700 & 0.634 & 0.302 & 0.100 & 0.668 \\
 & \ftc{} & 0.2 & 0.702 & 0.522 & 0.308 & 0.521 & 0.596 & 0.282 & 0.241 & 0.020 & 0.488 \\
 & Seq-SVD & 0.6 & 0.739 & 0.702 & 0.422 & 0.684 & 0.716 & 0.590 & 0.345 & 0.085 & 0.642 \\
 & Seq-SVD & 0.4 & 0.683 & 0.467 & 0.315 & 0.518 & 0.641 & 0.436 & 0.257 & 0.035 & 0.510 \\
 & Seq-SVD & 0.2 & 0.541 & 0.312 & 0.234 & 0.329 & 0.541 & 0.113 & 0.221 & 0.005 & 0.345 \\
 & enh-LeSTD & 0.8 & 0.786 & 0.796 & 0.531 & 0.773 & 0.731 & 0.713 & 0.371 & 0.345 & 0.722 \\
 & enh-LeSTD & 0.6 & 0.782 & 0.784 & 0.486 & 0.754 & 0.716 & 0.641 & 0.330 & 0.195 & 0.694 \\
 & enh-LeSTD & 0.4 & 0.739 & 0.694 & 0.409 & 0.658 & 0.650 & 0.458 & 0.287 & 0.040 & 0.601 \\
 & enh-LeSTD & 0.2 & 0.588 & 0.365 & 0.249 & 0.358 & 0.536 & 0.044 & 0.226 & 0.005 & 0.357 \\
 & TLLM-gen & 0.8 & 0.767 & 0.760 & 0.483 & 0.727 & 0.715 & 0.647 & 0.313 & 0.070 & 0.683 \\
 & TLLM-gen & 0.6 & 0.731 & 0.657 & 0.394 & 0.641 & 0.660 & 0.498 & 0.267 & 0.035 & 0.597 \\
 & TLLM-gen & 0.4 & 0.671 & 0.472 & 0.289 & 0.489 & 0.584 & 0.175 & 0.241 & 0.025 & 0.447 \\
 & TLLM-gen & 0.2 & 0.565 & 0.328 & 0.235 & 0.314 & 0.507 & 0.015 & 0.215 & 0.010 & 0.327 \\
 & MPIFA & 0.8 & 0.793 & 0.771 & 0.506 & 0.760 & 0.725 & 0.683 & 0.375 & 0.360 & 0.706 \\
 & MPIFA & 0.6 & 0.741 & 0.711 & 0.441 & 0.664 & 0.709 & 0.570 & 0.342 & 0.005 & 0.639 \\
 & MPIFA & 0.4 & 0.622 & 0.421 & 0.278 & 0.439 & 0.608 & 0.323 & 0.242 & 0.020 & 0.448 \\
 & MPIFA & 0.2 & 0.517 & 0.277 & 0.222 & 0.277 & 0.522 & 0.019 & 0.199 & 0.000 & 0.306 \\
 & SVD-LLM & 0.8 & 0.770 & 0.738 & 0.461 & 0.733 & 0.724 & 0.589 & 0.357 & 0.040 & 0.669 \\
 & SVD-LLM & 0.6 & 0.720 & 0.626 & 0.381 & 0.613 & 0.682 & 0.427 & 0.317 & 0.020 & 0.575 \\
 & SVD-LLM & 0.4 & 0.610 & 0.389 & 0.266 & 0.403 & 0.578 & 0.215 & 0.243 & 0.040 & 0.410 \\
 & SVD-LLM & 0.2 & 0.523 & 0.279 & 0.242 & 0.277 & 0.501 & 0.027 & 0.221 & 0.030 & 0.308 \\
 & Basis Sharing & 0.8 & 0.791 & 0.720 & 0.458 & 0.732 & 0.715 & 0.619 & 0.343 & 0.330 & 0.672 \\
 & Basis Sharing & 0.6 & 0.742 & 0.637 & 0.394 & 0.680 & 0.686 & 0.550 & 0.340 & 0.045 & 0.615 \\
 & Basis Sharing & 0.4 & 0.653 & 0.448 & 0.294 & 0.498 & 0.626 & 0.325 & 0.254 & 0.025 & 0.474 \\
 & Basis Sharing & 0.2 & 0.527 & 0.288 & 0.234 & 0.278 & 0.535 & 0.027 & 0.219 & 0.015 & 0.315 \\
\midrule
Qwen2.5-7B & dense & 1.0 & 0.791 & 0.769 & 0.511 & 0.781 & 0.730 & 0.720 & 0.446 & 0.820 & 0.717 \\
 & \ftc{} & 0.8 & 0.796 & 0.797 & 0.511 & 0.780 & 0.714 & 0.690 & 0.439 & 0.780 & 0.715 \\
 & \ftc{} & 0.6 & 0.790 & 0.727 & 0.492 & 0.770 & 0.707 & 0.650 & 0.451 & 0.715 & 0.689 \\
 & \ftc{} & 0.4 & 0.785 & 0.732 & 0.486 & 0.742 & 0.676 & 0.535 & 0.395 & 0.445 & 0.659 \\
 & \ftc{} & 0.2 & 0.614 & 0.410 & 0.258 & 0.423 & 0.565 & 0.163 & 0.227 & 0.015 & 0.405 \\
 & Seq-SVD & 0.6 & 0.745 & 0.690 & 0.442 & 0.688 & 0.707 & 0.514 & 0.359 & 0.445 & 0.631 \\
 & Seq-SVD & 0.4 & 0.539 & 0.297 & 0.209 & 0.322 & 0.528 & 0.095 & 0.237 & 0.010 & 0.332 \\
 & Seq-SVD & 0.2 & 0.543 & 0.264 & 0.242 & 0.255 & 0.524 & 0.000 & 0.237 & 0.010 & 0.305 \\
 & enh-LeSTD & 0.8 & 0.787 & 0.777 & 0.497 & 0.767 & 0.709 & 0.639 & 0.426 & 0.705 & 0.696 \\
 & enh-LeSTD & 0.6 & 0.779 & 0.727 & 0.474 & 0.733 & 0.693 & 0.543 & 0.419 & 0.495 & 0.658 \\
 & enh-LeSTD & 0.4 & 0.673 & 0.533 & 0.322 & 0.551 & 0.622 & 0.240 & 0.306 & 0.015 & 0.490 \\
 & enh-LeSTD & 0.2 & 0.526 & 0.268 & 0.247 & 0.258 & 0.491 & 0.000 & 0.220 & 0.000 & 0.298 \\
 & TLLM-gen & 0.8 & 0.782 & 0.776 & 0.507 & 0.750 & 0.693 & 0.521 & 0.398 & 0.645 & 0.671 \\
 & TLLM-gen & 0.6 & 0.531 & 0.321 & 0.243 & 0.313 & 0.528 & 0.021 & 0.281 & 0.010 & 0.326 \\
 & TLLM-gen & 0.4 & 0.531 & 0.263 & 0.247 & 0.281 & 0.481 & 0.000 & 0.238 & 0.015 & 0.300 \\
 & TLLM-gen & 0.2 & 0.518 & 0.251 & 0.242 & 0.258 & 0.514 & 0.000 & 0.191 & 0.005 & 0.297 \\
 & MPIFA & 0.8 & 0.784 & 0.752 & 0.487 & 0.762 & 0.708 & 0.692 & 0.464 & 0.715 & 0.698 \\
 & MPIFA & 0.6 & 0.748 & 0.764 & 0.475 & 0.709 & 0.726 & 0.575 & 0.423 & 0.525 & 0.666 \\
 & MPIFA & 0.4 & 0.605 & 0.476 & 0.299 & 0.446 & 0.614 & 0.296 & 0.274 & 0.035 & 0.456 \\
 & MPIFA & 0.2 & 0.514 & 0.235 & 0.235 & 0.252 & 0.496 & 0.000 & 0.192 & 0.025 & 0.289 \\
 & SVD-LLM & 0.8 & 0.771 & 0.749 & 0.481 & 0.741 & 0.716 & 0.646 & 0.448 & 0.680 & 0.684 \\
 & SVD-LLM & 0.6 & 0.732 & 0.707 & 0.434 & 0.646 & 0.684 & 0.538 & 0.378 & 0.420 & 0.624 \\
 & SVD-LLM & 0.4 & 0.565 & 0.390 & 0.265 & 0.391 & 0.542 & 0.232 & 0.260 & 0.020 & 0.398 \\
 & SVD-LLM & 0.2 & 0.516 & 0.259 & 0.243 & 0.260 & 0.500 & 0.004 & 0.234 & 0.015 & 0.297 \\
 & Basis Sharing & 0.8 & 0.796 & 0.800 & 0.529 & 0.765 & 0.715 & 0.688 & 0.443 & 0.720 & 0.716 \\
 & Basis Sharing & 0.6 & 0.776 & 0.742 & 0.472 & 0.723 & 0.711 & 0.589 & 0.410 & 0.535 & 0.669 \\
 & Basis Sharing & 0.4 & 0.662 & 0.582 & 0.374 & 0.550 & 0.602 & 0.364 & 0.294 & 0.090 & 0.522 \\
 & Basis Sharing & 0.2 & 0.538 & 0.267 & 0.219 & 0.250 & 0.494 & 0.009 & 0.239 & 0.010 & 0.296 \\
 
\bottomrule
\end{tabular}
\end{table}

\begin{table}[H]
\ContinuedFloat
\centering\tiny
\caption{Complete downstream grid, GQA and hybrid models (continued).}
\setlength{\tabcolsep}{2pt}
\begin{tabular}{lllccccccccc}
\toprule
Model & Method & $\rrho$ & PIQA & ARC-e & ARC-c & HSwag & WinoG
& LAMB & MathQA & GSM8K & avg6 \\
\midrule
Qwen3-8B & dense & 1.0 & 0.786 & 0.801 & 0.561 & 0.774 & 0.724 & 0.704 & 0.562 & 0.855 & 0.725 \\
 & \ftc{} & 0.8 & 0.788 & 0.832 & 0.583 & 0.763 & 0.735 & 0.664 & 0.563 & 0.765 & 0.728 \\
 & \ftc{} & 0.6 & 0.781 & 0.793 & 0.564 & 0.750 & 0.715 & 0.640 & 0.515 & 0.735 & 0.707 \\
 & \ftc{} & 0.4 & 0.730 & 0.641 & 0.441 & 0.712 & 0.670 & 0.553 & 0.316 & 0.105 & 0.624 \\
 & \ftc{} & 0.2 & 0.690 & 0.550 & 0.324 & 0.550 & 0.582 & 0.305 & 0.247 & 0.005 & 0.500 \\
 & Seq-SVD & 0.6 & 0.765 & 0.712 & 0.487 & 0.718 & 0.699 & 0.627 & 0.423 & 0.455 & 0.668 \\
 & Seq-SVD & 0.4 & 0.707 & 0.567 & 0.373 & 0.606 & 0.642 & 0.472 & 0.284 & 0.085 & 0.561 \\
 & Seq-SVD & 0.2 & 0.557 & 0.320 & 0.209 & 0.305 & 0.534 & 0.129 & 0.245 & 0.005 & 0.342 \\
 & enh-LeSTD & 0.8 & 0.789 & 0.831 & 0.590 & 0.753 & 0.718 & 0.617 & 0.547 & 0.790 & 0.716 \\
 & enh-LeSTD & 0.6 & 0.773 & 0.760 & 0.493 & 0.684 & 0.664 & 0.489 & 0.454 & 0.540 & 0.644 \\
 & enh-LeSTD & 0.4 & 0.724 & 0.609 & 0.388 & 0.558 & 0.604 & 0.294 & 0.307 & 0.070 & 0.529 \\
 & enh-LeSTD & 0.2 & 0.555 & 0.348 & 0.238 & 0.305 & 0.514 & 0.016 & 0.215 & 0.005 & 0.329 \\
 & TLLM-gen & 0.8 & 0.777 & 0.761 & 0.499 & 0.706 & 0.700 & 0.519 & 0.468 & 0.635 & 0.660 \\
 & TLLM-gen & 0.6 & 0.749 & 0.671 & 0.444 & 0.629 & 0.647 & 0.373 & 0.393 & 0.235 & 0.586 \\
 & TLLM-gen & 0.4 & 0.663 & 0.495 & 0.320 & 0.454 & 0.550 & 0.109 & 0.286 & 0.010 & 0.432 \\
 & TLLM-gen & 0.2 & 0.532 & 0.297 & 0.215 & 0.282 & 0.500 & 0.001 & 0.210 & 0.000 & 0.304 \\
 & MPIFA & 0.8 & 0.791 & 0.773 & 0.570 & 0.780 & 0.714 & 0.626 & 0.484 & 0.695 & 0.709 \\
 & MPIFA & 0.6 & 0.773 & 0.722 & 0.505 & 0.740 & 0.714 & 0.564 & 0.450 & 0.440 & 0.670 \\
 & MPIFA & 0.4 & 0.684 & 0.531 & 0.330 & 0.536 & 0.644 & 0.406 & 0.279 & 0.080 & 0.522 \\
 & MPIFA & 0.2 & 0.554 & 0.333 & 0.217 & 0.311 & 0.533 & 0.199 & 0.225 & 0.000 & 0.358 \\
 & SVD-LLM & 0.8 & 0.777 & 0.752 & 0.529 & 0.762 & 0.715 & 0.542 & 0.436 & 0.335 & 0.679 \\
 & SVD-LLM & 0.6 & 0.750 & 0.577 & 0.416 & 0.690 & 0.688 & 0.488 & 0.373 & 0.200 & 0.601 \\
 & SVD-LLM & 0.4 & 0.624 & 0.446 & 0.255 & 0.393 & 0.555 & 0.278 & 0.246 & 0.010 & 0.425 \\
 & SVD-LLM & 0.2 & 0.520 & 0.302 & 0.214 & 0.276 & 0.511 & 0.106 & 0.226 & 0.005 & 0.322 \\
 & Basis Sharing & 0.8 & 0.774 & 0.769 & 0.566 & 0.765 & 0.716 & 0.561 & 0.458 & 0.430 & 0.692 \\
 & Basis Sharing & 0.6 & 0.775 & 0.746 & 0.531 & 0.724 & 0.707 & 0.482 & 0.406 & 0.200 & 0.661 \\
 & Basis Sharing & 0.4 & 0.737 & 0.587 & 0.402 & 0.632 & 0.638 & 0.369 & 0.330 & 0.045 & 0.561 \\
 & Basis Sharing & 0.2 & 0.544 & 0.289 & 0.207 & 0.306 & 0.529 & 0.137 & 0.239 & 0.005 & 0.335 \\
\midrule
Qwen3-32B & dense & 1.0 & 0.802 & 0.829 & 0.608 & 0.818 & 0.729 & 0.671 & 0.611 & 0.885 & 0.743 \\
 & \ftc{} & 0.6 & 0.801 & 0.815 & 0.594 & 0.811 & 0.723 & 0.656 & 0.536 & 0.770 & 0.733 \\
 & \ftc{} & 0.4 & 0.799 & 0.812 & 0.617 & 0.795 & 0.739 & 0.635 & 0.487 & 0.750 & 0.733 \\
 & \ftc{} & 0.2 & 0.782 & 0.777 & 0.556 & 0.727 & 0.696 & 0.489 & 0.359 & 0.280 & 0.671 \\
 & Seq-SVD & 0.6 & 0.784 & 0.809 & 0.607 & 0.786 & 0.728 & 0.666 & 0.483 & 0.635 & 0.730 \\
 & Seq-SVD & 0.4 & 0.768 & 0.762 & 0.536 & 0.728 & 0.714 & 0.565 & 0.380 & 0.545 & 0.679 \\
 & Seq-SVD & 0.2 & 0.582 & 0.380 & 0.253 & 0.370 & 0.530 & 0.134 & 0.233 & 0.000 & 0.375 \\
 & enh-LeSTD & 0.6 & 0.799 & 0.811 & 0.596 & 0.784 & 0.727 & 0.587 & 0.532 & 0.765 & 0.717 \\
 & enh-LeSTD & 0.4 & 0.786 & 0.784 & 0.595 & 0.751 & 0.734 & 0.441 & 0.469 & 0.510 & 0.682 \\
 & enh-LeSTD & 0.2 & 0.685 & 0.494 & 0.356 & 0.542 & 0.593 & 0.092 & 0.281 & 0.015 & 0.460 \\
 & TLLM-gen & 0.6 & 0.795 & 0.788 & 0.580 & 0.767 & 0.728 & 0.463 & 0.505 & 0.700 & 0.687 \\
 & TLLM-gen & 0.4 & 0.785 & 0.718 & 0.513 & 0.708 & 0.706 & 0.278 & 0.409 & 0.210 & 0.618 \\
 & TLLM-gen & 0.2 & 0.594 & 0.358 & 0.268 & 0.366 & 0.530 & 0.017 & 0.228 & 0.000 & 0.355 \\
 & MPIFA & 0.6 & 0.783 & 0.788 & 0.632 & 0.779 & 0.724 & 0.628 & 0.475 & 0.715 & 0.722 \\
 & MPIFA & 0.4 & 0.734 & 0.711 & 0.501 & 0.657 & 0.665 & 0.504 & 0.370 & 0.405 & 0.629 \\
 & MPIFA & 0.2 & 0.594 & 0.335 & 0.246 & 0.341 & 0.517 & 0.200 & 0.230 & 0.015 & 0.372 \\
 & SVD-LLM & 0.6 & 0.765 & 0.745 & 0.566 & 0.751 & 0.698 & 0.579 & 0.426 & 0.665 & 0.684 \\
 & SVD-LLM & 0.4 & 0.708 & 0.587 & 0.380 & 0.523 & 0.605 & 0.348 & 0.293 & 0.035 & 0.525 \\
 & SVD-LLM & 0.2 & 0.545 & 0.303 & 0.218 & 0.287 & 0.521 & 0.135 & 0.242 & 0.000 & 0.335 \\
 & Basis Sharing & 0.6 & 0.781 & 0.753 & 0.579 & 0.771 & 0.723 & 0.600 & 0.479 & 0.600 & 0.701 \\
 & Basis Sharing & 0.4 & 0.761 & 0.741 & 0.534 & 0.720 & 0.683 & 0.504 & 0.358 & 0.440 & 0.657 \\
 & Basis Sharing & 0.2 & 0.601 & 0.363 & 0.242 & 0.343 & 0.501 & 0.163 & 0.241 & 0.010 & 0.369 \\
\midrule
Qwen3.5-9B & dense & 1.0 & 0.794 & 0.776 & 0.565 & 0.770 & 0.734 & 0.696 & 0.553 & 0.885 & 0.723 \\
 & \ftc{} & 0.8 & 0.794 & 0.833 & 0.602 & 0.748 & 0.716 & 0.671 & 0.512 & 0.785 & 0.727 \\
 & \ftc{} & 0.6 & 0.755 & 0.822 & 0.550 & 0.681 & 0.705 & 0.625 & 0.421 & 0.580 & 0.690 \\
\bottomrule
\end{tabular}
\end{table}

\subsection{Hyperparameter ablations}
\label{app:hyper}

\paragraph{Ablation protocol and additional design studies.}
Table~\ref{tab:ablation} reports one-at-a-time ablations of the final
recipe on all four GQA models at $\rrho\in\{0.6,0.2\}$. The recipe uses
fixed head-concat tensorization, budget-coupled allocation, $a=0.5$,
$\alpha=0.02$, two refits, 256 calibration samples, slice-wise
normalization, no QKV teacher map, and a separate teacher-calibrated
$W_O$ target. Each variant changes one component while keeping the
remaining configuration and compression budget fixed. ``K/V repeat-expand''
replaces head-concat by naive GQA$\to$MHA replication, while ``dense core''
uses a dense core under the same parameter budget. Additional
development-stage studies below examine design choices not isolated by the
main grid.

\paragraph{QKV teacher-map regularization.}
Table~\ref{tab:qkv_teacher_map} shows that disabling the QKV teacher map consistently gives the lowest perplexity at $\rho=0.2$. Among the mapped variants, $\mu=3$ performs best on Mistral-7B, Llama-3-8B, and Qwen3-8B,
while Qwen2.5-7B favors $\mu=10$. Even at these best mapped settings, perplexity remains $2.1$--$11.3\%$ higher than the no-map configuration. This supports using the original QKV weights as the compression target in the final recipe, without an additional teacher-output map.
\begin{table}[H]
\centering\scriptsize
\caption{Sensitivity to QKV teacher-map regularization at $\rrho=0.2$.}
\label{tab:qkv_teacher_map}
\begin{tabular}{lccccccc}\toprule
$\rrho{=}0.2$, wt2 ppl  & $\mu{=}0.001$ & $0.03$ & $0.3$ & $1$ & $3$ & $10$ & no QKV map (final)\\ \midrule
Mistral-7B
& 8.10 & 8.13 & 8.04 & 7.85 & 7.64 & 7.91 & 7.33 \\
Llama-3-8B
& 12.40 & 12.35 & 11.87 & 11.32 & 11.18 & 11.37 & 10.44 \\
Qwen2.5-7B
& 16.52 & 16.80 & 17.42 & 16.84 & 14.72 & 13.97 & 12.55 \\
Qwen3-8B
& 12.26 & 12.09 & 11.38 & 11.07 & 10.93 & 11.61 & 10.70 \\ 
\bottomrule\end{tabular}
\end{table}

\paragraph{Rank allocation.}
Replacing the budget-coupled allocator with fixed Tucker ranks increases
perplexity by $3.7$--$6.2\%$ at $\rho=0.2$. The allocator therefore
provides a consistent but moderate gain, complementing the larger benefits
from the activation-induced metric, joint tensorization, and sequential
calibration.

\paragraph{Output-projection metric.}
Reusing the QKV activation metric for $W_O$ causes severe degradation,
increasing perplexity by $55$--$708\%$ at $\rho=0.6$ and by more than
$4000\%$ at $\rho=0.2$. This shows that $W_O$ requires a metric constructed
from its own attention-output representations rather than reusing the
input-side QKV metric.

\begin{table}[H]
\center\scriptsize
\caption{Perplexity ablations for rank allocation and the output-projection metric.}
\begin{tabular}{lllccc}\toprule
Model & $\rrho$ & Variant & variant & full & $\Delta$\\ \midrule
Mistral-7B
& 0.2 & fixed $R_1{=}640,R_2{=}128$
& 7.786 & 7.33 & $+6.2\%$ \\
Llama-3-8B
& 0.2 & fixed $R_1{=}640,R_2{=}128$
& 10.975 & 10.44 & $+5.1\%$ \\
Qwen3-8B
& 0.2 & fixed $R_1{=}640,R_2{=}128$
& 11.100 & 10.70 & $+3.7\%$ \\
\midrule
Mistral-7B
& 0.6 & $W_O$ with QKV metric
& 8.456 & 5.44 & $+55.4\%$ \\
Mistral-7B
& 0.2 & $W_O$ with QKV metric
& 303.397 & 7.33 & $+4039.1\%$ \\
Llama-3-8B
& 0.6 & $W_O$ with QKV metric
& 53.677 & 6.64 & $+708.4\%$ \\
Llama-3-8B
& 0.2 & $W_O$ with QKV metric
& 606.211 & 10.44 & $+5706.6\%$ \\
Qwen3-8B
& 0.6 & $W_O$ with QKV metric
& 12.865 & 7.36 & $+74.8\%$ \\
Qwen3-8B
& 0.2 & $W_O$ with QKV metric
& 454.125 & 10.70 & $+4144.2\%$ \\ 
\bottomrule
\end{tabular}
\end{table}

\begin{wraptable}{r}{0.46\columnwidth}
\vspace{-0.7em}
\centering
\caption{Perplexity sensitivity to one versus two refit iterations.}
\label{tab:refit_iterations}
\vspace{-0.5em}
\scriptsize
\setlength{\tabcolsep}{3.0pt}
\renewcommand{\arraystretch}{0.90}
\begin{tabular}{@{}lccc@{}}
\toprule
Model & Variant & Full & $\Delta$ \\
\midrule
Llama-3-8B  & 6.655 & 6.613 & 6.607 \\
Qwen3-8B    & 7.855 & 7.764 & 7.759 \\
Qwen2.5-7B  & 7.322 & 7.312 & 7.308 \\
\bottomrule
\end{tabular}
\vspace{-0.8em}
\end{wraptable}

\paragraph{Refit iterations.}
Table \ref{tab:refit_iterations} shows that one refit captures nearly all of the improvement, while a second refit
changes perplexity only marginally in the evaluated settings. We therefore
retain two refits in the final recipe.

\paragraph{Anchoring.}
The final recipe mixes the teacher Gram into the $W_O$ statistics with
$a=0.5$, while leaving the QKV statistics unanchored. As shown in
Table~\ref{tab:ablation}, anchoring has little effect at $\rrho=0.6$:
removing it changes perplexity by only $-0.6\%$ to $+0.8\%$ across the four
GQA models. Its contribution becomes more pronounced at $\rrho=0.2$, where
removing the anchor increases perplexity by $1.5\%$ on Mistral-7B,
$4.1\%$ on Llama-3-8B, $14.5\%$ on Qwen2.5-7B, and $6.1\%$ on Qwen3-8B.
The effect is therefore mainly visible under aggressive compression, with
Qwen2.5-7B showing the largest sensitivity.

\paragraph{Slice-wise normalization.}
The effect of slice-wise normalization is strongly budget-dependent.
As shown in Table~\ref{tab:ablation}, removing the slice-wise norm has only
a limited and model-dependent effect at $\rrho=0.6$, changing perplexity by
$-0.3\%$ to $+5.5\%$ across the four GQA models. At $\rrho=0.2$, however,
the degradation becomes substantially larger: $27.2\%$ on Mistral-7B,
$89.4\%$ on Llama-3-8B, $492.4\%$ on Qwen2.5-7B, and $11.8\%$ on
Qwen3-8B. This indicates that slice-wise normalization becomes increasingly
important under aggressive compression, where a tighter core budget makes
the decomposition more sensitive to energy imbalance across head slices.

\subsection{Uniform recipe, calibration source, seeds, structured core}
\label{app:uniform}

\paragraph{The final recipe at every ratio.}
Table~\ref{tab:uniform} reports the fixed recipe of
Section~\ref{sec:experiments} (allocator, $a=0.5$, no QKV teacher map,
$W_O$ map with $\mu=1$, $\alpha=0.02$, two refits, 256 calibration
windows, and slice-wise norm) at all four compression ratios. The same
recipe is used across models and ratios without model-specific tuning.
\begin{table}[H]
\centering\scriptsize
\caption{The final recipe on the six primary evaluation models at all four keep ratios. WikiText-2 perplexity, avg6, and GSM8K accuracy. GSM8K was not evaluated for GPT-J-6B and
Llama-2-13B at $\rrho=0.2$.  }
\label{tab:uniform}
\begin{tabular}{l|ccc|ccc|ccc|ccc}\toprule
& \multicolumn{3}{c|}{$\rho= 0.8$} & \multicolumn{3}{c|}{$\rho= 0.6$} & \multicolumn{3}{c|}{$\rho= 0.4$} & \multicolumn{3}{c}{$\rho= 0.2$}\\
Model & ppl & avg6 & GSM8K & ppl & avg6 & GSM8K & ppl & avg6 & GSM8K & ppl & avg6 & GSM8K\\ \midrule
GPT-J-6B & 9.04 & 0.613 & 0.010 & 9.30 & 0.598 & 0.005 & 10.03 & 0.564 & 0.025 & 13.60 & 0.448 & n/a \\
Llama-2-13B & 4.98 & 0.715 & 0.190 & 5.11 & 0.703 & 0.115 & 5.52 & 0.660 & 0.030 & 7.52 & 0.497 & n/a \\
Mistral-7B & 5.36 & 0.733 & 0.235 & 5.44 & 0.726 & 0.240 & 5.70 & 0.689 & 0.105 & 7.33 & 0.512 & 0.015 \\
Llama-3-8B & 6.36 & 0.729 & 0.395 & 6.64 & 0.716 & 0.290 & 7.31 & 0.668 & 0.100 & 10.44 & 0.488 & 0.020 \\
Qwen2.5-7B & 6.99 & 0.715 & 0.780 & 7.20 & 0.689 & 0.715 & 7.84 & 0.659 & 0.445 & 12.55 & 0.405 & 0.015 \\
Qwen3-8B & 7.12 & 0.728 & 0.765 & 7.36 & 0.707 & 0.735 & 8.48 & 0.624 & 0.105 & 10.70 & 0.500 & 0.005 \\ 
\bottomrule\end{tabular}
\end{table}

\paragraph{Whitening metric and damping.}
We use
$H_{\epsilon, l}=C_{s,l}+\epsilon\bar{t}_{l} I$, where
$\bar{t} =\frac{\operatorname{tr}(C_{s,l})}{d_{\mathrm{model}}}$., with
$\epsilon=10^{-4}$ throughout.
Table~\ref{tab:damping} examines sensitivity to this damping
parameter. At $\rrho=0.2$, increasing $\epsilon$ to $10^{-2}$ raises
perplexity by only $1.4$--$2.7\%$, whereas the heavily damped setting
$\epsilon=1$ incurs $11.7$--$24.8\%$ higher perplexity across the three
models. At $\rrho=0.6$, the effect is much smaller and model-dependent,
with changes ranging from $-1.5\%$ to $+3.0\%$.
These results indicate that aggressive compression benefits from preserving
the original functional geometry rather than strongly damping the whitening
metric. We therefore fix $\epsilon=10^{-4}$ across all models and
compression ratios.
\begin{table}[H]
\centering\scriptsize
\caption{Sensitivity to whitening damping $\epsilon$ in
$H_{\epsilon, l}=C_{s,l}+\epsilon\bar{t}_{l} I$, where
$\bar{t} =\frac{\operatorname{tr}(C_{s,l})}{d_{\mathrm{model}}}$. The recipe uses
$\epsilon=10^{-4}$.}
\label{tab:damping}
\begin{tabular}{llccc}\toprule
Model & $\rrho$ & $\varepsilon{=}10^{-4}$ (recipe) & $\varepsilon{=}10^{-2}$ & $\varepsilon{=}1$\\ \midrule
Llama-3-8B & .6 & 6.64 & 6.63 ($-0.1\%$) & 6.54 ($-1.5\%$) \\
Llama-3-8B & .2 & 10.44 & 10.59 ($+1.4\%$) & 12.21 ($+17.0\%$) \\
Mistral-7B & .6 & 5.44 & 5.44 ($-0.0\%$) & 5.45 ($+0.2\%$) \\
Mistral-7B & .2 & 7.33 & 7.43 ($+1.4\%$) & 8.19 ($+11.7\%$) \\
Qwen3-8B   & .6 & 7.36 & 7.39 ($+0.4\%$) & 7.58 ($+3.0\%$) \\
Qwen3-8B   & .2 & 10.70 & 10.99 ($+2.7\%$) & 13.36 ($+24.8\%$) \\ 
\bottomrule\end{tabular}
\end{table}

\paragraph{Whitening stability guard.}
The default whitening metric uses $\epsilon=10^{-4}$. To ensure numerical
stability, we increase $\epsilon$ by factors of ten whenever the de-whitened
reconstruction is non-finite or its norm exceeds $10\times$ that of the
original weight. The same guard is applied uniformly to all configurations.
It is never activated at $\rrho\geq0.4$, and no run falls back to the
identity metric or leaves any layer uncompressed.

At $\rrho=0.2$, the guard is triggered only for a subset of layers
(Table~\ref{tab:escalation}). The most sensitive case is Qwen2.5-7B: under the
final no-QKV-map recipe, 9 of 28 layers require increased damping, with 3
reaching $\epsilon=1$. The escalation frequency also changes with the QKV
target, indicating that different compression choices can alter the
conditioning of the activation statistics seen by subsequent layers.
Together with the damping sweep in Table~\ref{tab:damping}, where
larger $\epsilon$ generally degrades performance at this budget, these
results show that the guard serves as a numerical safeguard rather than a
mechanism for improving compression quality.

\begin{table}[H]
\centering\scriptsize
\caption{Layers whose whitening damping had to be escalated above
$\varepsilon{=}10^{-4}$ at $\rrho{=}0.2$, and how many of those ended at the
ridge metric $\varepsilon{=}1$. Parsed from the run logs. At $\rrho\ge0.4$
all counts are zero on all four models.}
\label{tab:escalation}
\begin{tabular}{llcc}\toprule
Model & QKV target & layers escalated & reaching $\varepsilon{=}1$\\ \midrule
Mistral-7B & no QKV map (final) & 4/32 & 0 \\
Mistral-7B & $\mu{=}1$ map & 5/32 & 1 \\
Llama-3-8B & no QKV map (final) & 2/32 & 0 \\
Llama-3-8B & $\mu{=}1$ map & 2/32 & 0 \\
Qwen2.5-7B & no QKV map (final) & 9/28 & 3 \\
Qwen2.5-7B & $\mu{=}1$ map & 5/28 & 0 \\
Qwen3-8B & no QKV map (final) & 1/36 & 0 \\
Qwen3-8B & $\mu{=}1$ map & 7/36 & 1 \\ 
\bottomrule\end{tabular}
\end{table}

\paragraph{Calibration-set size.}
The data-dependent baselines use 128 calibration sequences following
their published configurations, whereas FTC uses 256 sequences in the
main experiments. Table~\ref{tab:calibsize} evaluates the sensitivity
of FTC to this difference. At $\rho=0.6$, reducing the calibration set
from 256 to 128 windows changes perplexity by only $-0.1$--$+0.3\%$
across all four models. At $\rho=0.2$, the effect remains modest on
Mistral-7B and Llama-3-8B ($+2.3\%$ and $+3.7\%$, respectively), and is
negligible on Qwen3-8B ($-0.2\%$). Qwen2.5-7B is a notable exception:
its perplexity increases from $12.55$ to $43.90$ when the calibration
set is reduced to 128 windows. These results indicate that calibration
size has little effect in most settings, but can become important for
particular models under aggressive compression.
\begin{table}[H]
\centering\scriptsize
\caption{
 Sensitivity of \ftc{} to calibration-set size. We compare the
same final recipe using 256 versus 128 WikiText-2 calibration windows
of 2,048 tokens.}
\label{tab:calibsize}
\begin{tabular}{llccc}\toprule
Model & $\rrho$ & 256 windows & 128 windows & $\Delta$\\ \midrule
Llama-3-8B & 0.6 & 6.640 & 6.660 & $+0.3\%$ \\
Llama-3-8B & 0.2 & 10.440 & 10.821 & $+3.7\%$ \\
Mistral-7B & 0.6 & 5.440 & 5.447 & $+0.1\%$ \\
Mistral-7B & 0.2 & 7.330 & 7.499 & $+2.3\%$ \\
Qwen2.5-7B & 0.6 & 7.200 & 7.192 & -0.1\% \\
Qwen2.5-7B & 0.2 & 12.550 & 43.901 & +249.8\% \\
Qwen3-8B & 0.6 & 7.360 & 7.352 & -0.1\% \\
Qwen3-8B & 0.2 & 10.700 & 10.678 & -0.2\% \\
 
\bottomrule\end{tabular}
\end{table}

\paragraph{Structured (2:4) core.}
Table~\ref{tab:p24} evaluates a hardware-aligned 2:4 core against the
default unstructured sparse core under the corresponding realised budget.
At moderate compression, the additional quality cost of the structured
pattern is small: the perplexity increase is only $0.3$--$1.3\%$ at
$\rho=0.6$ and $0.8$--$2.0\%$ at $\rho=0.4$ across the four models.
The effect becomes substantially more model-dependent at $\rho=0.2$:
perplexity increases by $11.3\%$ on Qwen3-8B, $16.5\%$ on Llama-3-8B,
$17.9\%$ on Mistral-7B, and $366.2\%$ on Qwen2.5-7B. In particular,
Qwen2.5-7B increases from $12.55$ to $58.50$ perplexity under the 2:4
constraint, whereas its unstructured FTC checkpoint at the same nominal
budget is $1.83\times$ the dense perplexity ($12.55$ vs.\ $6.85$).
Thus, the 2:4 structure is nearly cost-free at moderate compression but
can impose a substantial and highly model-dependent penalty under the
tightest budget.

\begin{table}[H]
\centering\scriptsize
\caption{Unstructured vs.\ 2:4 core at equal $\rrho$. Both columns are the
final recipe at a realised $\rrho$, differing only in the core structure.}
\label{tab:p24}
\begin{tabular}{@{}lcccc@{}}
\toprule
Model & $\rrho$ & Unstructured core & 2:4 core & $\Delta_{\rm ppl}$ \\
\midrule
Mistral-7B & 0.6 & 5.44 & 5.47 & $+0.6\%$ \\
Mistral-7B & 0.4 & 5.70 & 5.76 & $+1.1\%$ \\
Mistral-7B & 0.2 & 7.33 & 8.64 & $+17.9\%$ \\

Llama-3-8B & 0.6 & 6.64 & 6.72 & $+1.2\%$ \\
Llama-3-8B & 0.4 & 7.31 & 7.46 & $+2.0\%$ \\
Llama-3-8B & 0.2 & 10.44 & 12.16 & $+16.5\%$ \\

Qwen2.5-7B & 0.6 & 7.20 & 7.22 & $+0.3\%$ \\
Qwen2.5-7B & 0.4 & 7.84 & 7.90 & $+0.8\%$ \\
Qwen2.5-7B & 0.2 & 12.55 & 58.50 & $+366.2\%$ \\

Qwen3-8B & 0.6 & 7.36 & 7.46 & $+1.3\%$ \\
Qwen3-8B & 0.4 & 8.48 & 8.56 & $+0.9\%$ \\
Qwen3-8B & 0.2 & 10.70 & 11.91 & $+11.3\%$ \\ 
\bottomrule\end{tabular}
\end{table}

\paragraph{Calibration-source sensitivity.}
Table~\ref{tab:calib-c4-baselines}  compares WikiText-2 and C4 calibration for the data-dependent methods while keeping each method's calibration budget fixed. C4 calibration consistently improves C4 perplexity, while WikiText-2 perplexity generally increases. The effect is mild at
$\rrho=0.6$ but becomes substantially larger at $\rrho=0.2$, indicating greater sensitivity to the calibration distribution under aggressive compression. PTB shows a less uniform trend. Qwen2.5-7B is a notable
exception at $\rrho=0.2$, where the calibration source also affects numerical stability. We therefore use WikiText-2 as the calibration source for all main data-dependent comparisons and analyze calibration-set size
separately.

\begin{table}[H]
\centering\scriptsize
\caption{Data-dependent methods under WikiText-2 and C4 calibration.}
\label{tab:calib-c4-baselines}
\begin{tabular}{llll|ccc}\toprule
Model & $\rrho$ & Method & Calibration & wt2 & PTB & C4\\ \midrule
Mistral-7B & 0.6 & SVD-LLM & WikiText-2 & 5.92 & 48.17 & 10.76 \\
Mistral-7B & 0.6 & SVD-LLM & C4 & 6.59 & 46.76 & 9.74 \\
Mistral-7B & 0.6 & MPIFA & WikiText-2 & 5.63 & 41.47 & 9.78 \\
Mistral-7B & 0.6 & MPIFA & C4 & 6.01 & 41.09 & 9.17 \\
Mistral-7B & 0.6 & Basis & WikiText-2 & 5.75 & 46.79 & 10.11 \\
Mistral-7B & 0.6 & Basis & C4 & 6.21 & 46.32 & 9.38 \\
Mistral-7B & 0.6 & \ftc{} & WikiText-2 & 5.44 & 43.16 & 8.91 \\
Mistral-7B & 0.6 & \ftc{} & C4 & 5.58 & 39.23 & 8.77 \\

\midrule
Mistral-7B & 0.2 & SVD-LLM & WikiText-2 & 25.40 & 479.05 & 82.38 \\
Mistral-7B & 0.2 & SVD-LLM & C4 & 81.34 & 564.41 & 55.14 \\
Mistral-7B & 0.2 & MPIFA & WikiText-2 & 15.92 & 299.66 & 51.15 \\
Mistral-7B & 0.2 & MPIFA & C4 & 45.94 & 274.50 & 30.73 \\
Mistral-7B & 0.2 & Basis & WikiText-2 & 16.18 & 350.76 & 54.43 \\
Mistral-7B & 0.2 & Basis & C4 & 48.04 & 285.01 & 33.74 \\
Mistral-7B & 0.2 & \ftc{} & WikiText-2 & 7.33 & 114.66 & 15.91 \\
Mistral-7B & 0.2 & \ftc{} & C4 & 12.41 & 99.54 & 12.19 \\

\midrule
Llama-3-8B & 0.6 & SVD-LLM & WikiText-2 & 8.02 & 17.51 & 38.23 \\
Llama-3-8B & 0.6 & SVD-LLM & C4 & 10.48 & 17.73 & 13.50 \\
Llama-3-8B & 0.6 & MPIFA & WikiText-2 & 7.26 & 14.70 & 20.53 \\
Llama-3-8B & 0.6 & MPIFA & C4 & 8.16 & 14.38 & 11.37 \\
Llama-3-8B & 0.6 & Basis & WikiText-2 & 7.60 & 15.92 & 15.03 \\
Llama-3-8B & 0.6 & Basis & C4 & 9.25 & 15.83 & 12.49 \\
Llama-3-8B & 0.6 & \ftc{} & WikiText-2 & 6.64 & 12.44 & 11.01 \\
Llama-3-8B & 0.6 & \ftc{} & C4 & 6.93 & 12.25 & 10.31 \\

\midrule
Llama-3-8B & 0.2 & SVD-LLM & WikiText-2 & 48.09 & 264.41 & 253.66 \\
Llama-3-8B & 0.2 & SVD-LLM & C4 & 212.86 & 305.31 & 117.94 \\
Llama-3-8B & 0.2 & MPIFA & WikiText-2 & 27.02 & 128.53 & 132.05 \\
Llama-3-8B & 0.2 & MPIFA & C4 & 92.45 & 107.70 & 51.89 \\
Llama-3-8B & 0.2 & Basis & WikiText-2 & 35.60 & 217.80 & 180.43 \\
Llama-3-8B & 0.2 & Basis & C4 & 176.36 & 239.78 & 106.65 \\
Llama-3-8B & 0.2 & \ftc{} & WikiText-2 & 10.44 & 50.71 & 29.18 \\
Llama-3-8B & 0.2 & \ftc{} & C4 & 23.77 & 43.20 & 19.42 \\

\midrule
Qwen2.5-7B & 0.6 & SVD-LLM & WikiText-2 & 8.56 & 19.39 & 22.11 \\
Qwen2.5-7B & 0.6 & SVD-LLM & C4 & 12.40 & 17.56 & 15.38 \\
Qwen2.5-7B & 0.6 & MPIFA & WikiText-2 & 7.84 & 16.93 & 19.00 \\
Qwen2.5-7B & 0.6 & MPIFA & C4 & 9.63 & 15.70 & 13.70 \\
Qwen2.5-7B & 0.6 & Basis & WikiText-2 & 7.82 & 16.79 & 16.06 \\
Qwen2.5-7B & 0.6 & Basis & C4 & 8.89 & 15.77 & 13.54 \\
Qwen2.5-7B & 0.6 & \ftc{} & WikiText-2 & 7.20 & 14.26 & 12.50 \\
Qwen2.5-7B & 0.6 & \ftc{} & C4 & 7.55 & 14.00 & 12.23 \\

\midrule
Qwen2.5-7B & 0.2 & SVD-LLM & WikiText-2 & 255.23 & 3251 & 1406 \\
Qwen2.5-7B & 0.2 & SVD-LLM & C4 & 1020 & 2547 & 450.36 \\
Qwen2.5-7B & 0.2 & MPIFA & WikiText-2 & 7791 & 55759 & 26390 \\
Qwen2.5-7B & 0.2 & MPIFA & C4 & 341.05 & 688.78 & 161.15 \\
Qwen2.5-7B & 0.2 & Basis & WikiText-2 & 140.65 & 1760 & 786.30 \\
Qwen2.5-7B & 0.2 & Basis & C4 & 854.58 & 1152 & 315.74 \\
Qwen2.5-7B & 0.2 & \ftc{} & WikiText-2 & 12.55 & 89.42 & 73.37 \\
Qwen2.5-7B & 0.2 & \ftc{} & C4 & \multicolumn{3}{c}{numerically unstable$^\dagger$} \\

\midrule
Qwen3-8B & 0.6 & SVD-LLM & WikiText-2 & 9.24 & 19.69 & 17.61 \\
Qwen3-8B & 0.6 & SVD-LLM & C4 & 10.45 & 20.15 & 16.39 \\
Qwen3-8B & 0.6 & MPIFA & WikiText-2 & 8.17 & 16.89 & 15.18 \\
Qwen3-8B & 0.6 & MPIFA & C4 & 8.72 & 16.15 & 13.83 \\
Qwen3-8B & 0.6 & Basis & WikiText-2 & 9.18 & 18.26 & 16.70 \\
Qwen3-8B & 0.6 & Basis & C4 & 9.83 & 17.78 & 15.35 \\
Qwen3-8B & 0.6 & \ftc{} & WikiText-2 & 7.36 & 13.58 & 12.72 \\
Qwen3-8B & 0.6 & \ftc{} & C4 & 7.58 & 13.42 & 12.46 \\

\midrule
Qwen3-8B & 0.2 & SVD-LLM & WikiText-2 & 75.77 & 806.93 & 249.44 \\
Qwen3-8B & 0.2 & SVD-LLM & C4 & 185.24 & 497.86 & 145.08 \\
Qwen3-8B & 0.2 & MPIFA & WikiText-2 & 22.23 & 103.14 & 64.87 \\
Qwen3-8B & 0.2 & MPIFA & C4 & 52.09 & 83.33 & 43.36 \\
Qwen3-8B & 0.2 & Basis & WikiText-2 & 46.68 & 204.73 & 133.10 \\
Qwen3-8B & 0.2 & Basis & C4 & 117.24 & 171.98 & 104.44 \\
Qwen3-8B & 0.2 & \ftc{} & WikiText-2 & 10.70 & 33.02 & 23.73 \\
Qwen3-8B & 0.2 & \ftc{} & C4 & 15.72 & 26.82 & 18.87 \\
 
\bottomrule\end{tabular}
\end{table}

\subsection{Where drift accumulates along depth}
\label{app:depth}

\begin{figure}[H]
\centering
\includegraphics[width=\linewidth]{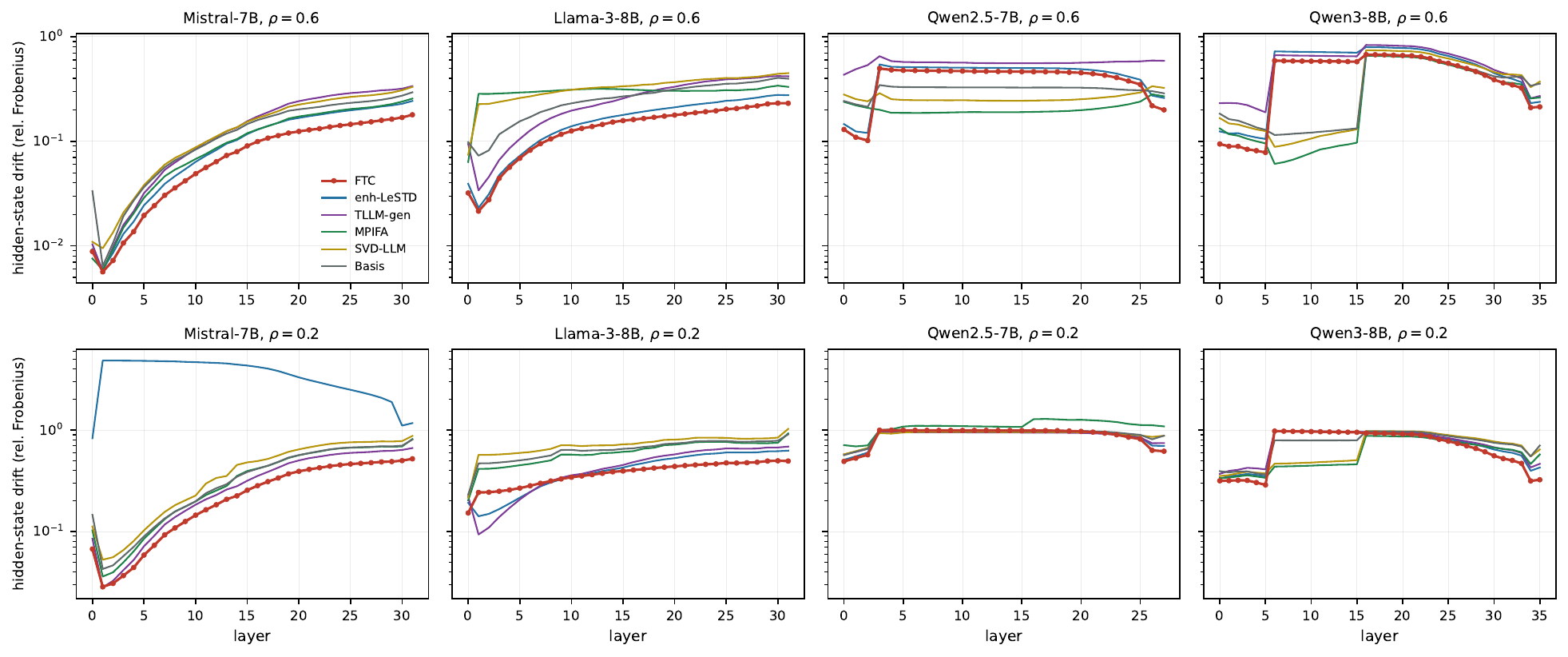}
\caption{
Layer-wise relative Frobenius drift of the residual stream between compressed
and dense models on 16 sequences disjoint from the calibration set. We compare \ftc{} with five representative methods on four GQA models at $\rrho\in\{0.6,0.2\}$.
Lower is better. \ftc{} generally maintains smaller drift across depth,
particularly under aggressive compression on Mistral-7B and Llama-3-8B.
On the Qwen models, the metric becomes less discriminative through the
middle layers at $\rrho=0.2$, with clearer separation near the output.
}
\label{fig:depth}
\end{figure}

\subsection{Robustness and Sensitivity}
\begin{wraptable}{r}{0.48\columnwidth}
\vspace{-0.6em}
\centering
\caption{Calibration-seed robustness of \ftc{}. Mean$\pm$std over three
calibration sets (seeds $\{42,1,2\}$).}
\label{tab:seed}
\vspace{-0.5em}
\scriptsize
\setlength{\tabcolsep}{2.2pt}
\renewcommand{\arraystretch}{0.88}
\begin{tabular}{@{}lccc@{}}
\toprule
Model & WT2 ($\rrho=.6$) & WT2 ($\rrho=.2$) & avg6 ($\rrho=.2$) \\
\midrule
Mistral-7B
& $5.44\pm0.00$
& $7.30\pm0.02$
& $0.511\pm0.005$ \\
Llama-3-8B
& $6.65\pm0.00$
& $10.45\pm0.03$
& $0.493\pm0.003$ \\
Qwen2.5-7B
& $7.20\pm0.01$
& $12.65\pm0.07$
& $0.405\pm0.002$ \\
Qwen3-8B
& $7.39\pm0.03$
& $10.70\pm0.02$
& $0.494\pm0.008$ \\
\bottomrule
\end{tabular}
\vspace{-0.8em}
\end{wraptable}

To assess the sensitivity of \ftc{} to calibration sampling, we repeat the
compression procedure with three independently sampled calibration sets.
Table~\ref{tab:seed} shows that perplexity is highly stable across all four
GQA models: the standard deviation is at most $0.03$ ppl in all displayed
settings except Qwen2.5-7B at the most aggressive budget, where it remains
$0.07$ ppl. Downstream avg6 is similarly stable at $\rrho=0.2$, with
standard deviations between $0.002$ and $0.008$. These results indicate limited sensitivity to calibration sampling across the tested seeds.

\section{Implementation and Optimization Details}
\label{app:implementation}

\subsection{Overall Compression Algorithm}
\label{app:overall_algorithm}

Figure~\ref{fig:workflow} provides a conceptual overview of the proposed framework, while Algorithm~\ref{alg:ftc} summarizes its complete layerwise implementation. The compression proceeds sequentially through the attention layers of the model. At each layer, activations from the
current student model  are  used to construct the activation-induced metric for joint Q/K/V compression. The resulting Q/K/V projections are inserted into the student model before the attention output projection  $W_{O,l}$ is processed,
so that its compression is based on the representations produced by the updated student attention branch. After replacing the corresponding projection weights with their compressed counterparts, the updated student model is used to generate the activations required for compressing the next layer.

After the Q/K/V projections of layer $l$ have been compressed, let $Z_l^{S}$ denote the input to the attention output projection $W_{O,l}$ produced by the updated student attention branch, and let $Z_l^{T}$ denote the corresponding input produced by the pretrained teacher model. These representations are used only for the separate compression of
$W_{O,l}$. The details of the activation-induced metric, shared Tucker representation, and budget-coupled rank--sparsity allocation are provided in Appendices~C--E, respectively, while sparse-core refinement and output-projection compression are detailed below.

\begin{algorithm}[t]
	\caption{Sequential Function-Driven Structured Compression}
	\label{alg:ftc}
	\begin{algorithmic}[1]
		\Require Pretrained teacher model $f^{T}$,
		calibration set $\mathcal D_{\mathrm{cal}}$,
		layerwise Q/K/V budgets $\{B_l\}_{l=1}^{L}$,
		and candidate rank sets $\{\mathcal R_l\}_{l=1}^{L}$.
		\Ensure Compressed student model $f^{S}$

		\State Initialize the student model as $f^{S}\leftarrow f^{T}$.

		\For{$l=1,\ldots,L$}

			\State Collect the current student inputs $X_{l}^{S}$ to attention layer $l$ using $\mathcal{D}_{\mathrm{cal}}$.

			\State Construct the activation-induced metric
			$H_{\varepsilon,l}$ from $X_l^{S}$

			\State Form the joint head-wise Q/K/V tensor
			$\mathcal{W}_{l}$ from
			$W_{Q,l}$, $W_{K,l}$, and $W_{V,l}$.

			\State Apply metric whitening and frontal-slice normalization to obtain the tensor used for structured approximation.
            
            \State Compute an HOSVD-based shared Tucker initialization for evaluating the candidate rank pairs.

            \For{each $\mathbf R_l=(R_{1,l},R_{2,l})
				\in\mathcal R_l$}
				\State Determine the admissible core cardinality $K_l(\mathbf{R}_{l})$ under budget $B_{l}$.
				\State Evaluate the corresponding projection core and allocation score $J_{l}(\mathbf{R}_{l})$.
			\EndFor

			\State Fix the Tucker factors at the selected rank $(R_{1,l}^{\star},R_{2,l}^{\star})$ and recompute the corresponding dense projection core.

			\State Gradually prune the projection core and perform support-constrained HOOI-style refinement until the target cardinality $K_{l}^\star$ is reached.

			\State Map the compressed Tucker representation back to the original weight coordinates and recover
			$\widehat{W}_{Q,l}$, $\widehat{W}_{K,l}$, and
			$\widehat{W}_{V,l}$

            \State Run the calibration sequences through the updated student attention branch to collect $Z_l^{S}$, and collect the corresponding
            teacher representation $Z_l^{T}$ from $f^{T}$.
            
            \State Construct the output-projection activation-induced metric from $Z_{l}^{S}$ and $Z_{l}^{T}$.
            
            \State Construct the calibrated target from $(Z_{l}^{S},Z_{l}^{T})$ when enabled.
            
            \State Compress $W_{O,l}$ separately at the prescribed compression ratio to obtain $\widehat{W}_{O,l}$.
            
            \State Replace $W_{O,l}$ by $\widehat{W}_{O,l}$ in $f^{S}$.
            \State Propagate the updated student to provide the representations used
            at layer $l+1$.
		\EndFor

		\State \Return $f^{S}$
	\end{algorithmic}
\end{algorithm}

\subsection{Gradual Pruning and Sparse-Core Refinement}
\label{app:core_refinement}
Once the rank--sparsity allocator selects
$(R_{1,l}^{\star},R_{2,l}^{\star},K_l^{\star})$, we further refine the corresponding Tucker representation before mapping it back to the original weight coordinates. The candidate Tucker subspaces used in the allocation stage are initialized by HOSVD, which provides a shared low-dimensional basis for evaluating different rank pairs. The one-step top-$K$ projection used by the allocator
serves only as a surrogate for estimating the sparsification cost of each candidate rank pair. The final sparse core is instead obtained progressively, allowing the surviving Tucker representation to adapt after each reduction in core cardinality.

Let $\overline{\mathcal G}_{l}^{(0)}$ denote the dense projection core
of the selected Tucker ranks in the balanced coordinates introduced in
Section~\ref{sec:budget_rank_sparsity}. We gradually reduce its number of active coefficients according to a decreasing sequence
\begin{equation*}
    K_{l}^{(0)} > K_{l}^{(1)}>
    \cdots > K_{l}^{(T)} = K_{l}^{\star}.
\end{equation*}

\paragraph{Gradual support reduction.}
The cardinality sequence is controlled by a  gradual core-pruning rate
$\alpha\in(0,1]$:
\begin{equation*}
    K_{l}^{(t+1)}  = \max\left\{
        K_{l}^{(t)} -  \left\lceil \alpha K_{l}^{(t)}\right\rceil,
        \, K_{l}^\star \right\}.
\end{equation*}
At stage $t$, we rank the currently active core coefficients by magnitude
and retain the $K_l^{(t+1)}$ largest ones. Denoting the resulting support
by $\Omega_l^{(t+1)}$, we obtain the nested sequence
\begin{equation*}
    \Omega_l^{(t+1)} \subseteq \Omega_l^{(t)},
    \qquad  |\Omega_l^{(t)}| = K_l^{(t)}.
\end{equation*}
For $0<\alpha<1$, the number of pruning stages is approximately
\begin{equation*}
    T  \approx
    \frac{ \log(K_{l}^\star/K_{l}^{(0)})}{\log(1-\alpha) }.
\end{equation*}
We use $\alpha=0.02$ in the main experiments, so each stage removes
approximately $2\%$ of the currently surviving core coefficients.
Setting $\alpha=1$ recovers one-shot magnitude pruning directly to
$K_l^\star$.

\paragraph{Latent-core sparsity.}
The pruning is applied entry-wise to the Tucker core rather than to the
reconstructed weight matrices. For the $j$-th frontal slice, the reconstructed projection can be written as
\begin{equation*}
    \overline{W}_{l,j}
    = \sum_{a=1}^{R_{1,l}^{\star}}
    \sum_{b=1}^{R_{2,l}^{\star}}
    \overline{\mathcal{G}}_{l}(a,b,j)\,u_{1,l}^{(a)}
    \left(u_{2,l}^{(b)}\right)^{\top}.
\end{equation*}
Each core entry $\overline{\mathcal{G}}_{l}(a,b,j)$ serves as the coefficient of a latent rank-one component
$u_{1,l}^{(a)}(u_{2,l}^{(b)})^{\top}$ in the reconstruction of the $j$-th frontal slice. Setting this coefficient to zero deactivates the corresponding latent component. As more core coefficients are pruned, fewer latent components remain active, and the Tucker core becomes increasingly sparse. Importantly, this sparsity acts in the latent Tucker
representation rather than directly in the original weight space. Since the Tucker factors are generally dense, sparsifying the core reduces the number of active latent components but does not necessarily introduce zeros into the reconstructed Q/K/V projection matrices.

\paragraph{Support-constrained refinement.}
After each support reduction, we refine the Tucker representation while
keeping both the selected ranks and the current support fixed. Let
$\Omega_l^{(t)}$ denote the active support at stage $t$. Factor updates are
allowed to change the shared Tucker subspaces, but coefficients outside
$\Omega_l^{(t)}$ are constrained to remain zero throughout the refinement.

For fixed factors, the dense projection core is $\overline{\mathcal{G}}_{\mathrm{proj},l}=\overline{\mathcal{W}}_{l}
    \times_1 \widetilde{U}_{1,l}^{\top}\times_2 U_{2,l}^{\top}.$
We then restrict this core to the current support
\begin{equation*}
    \overline{\mathcal G}_{l} = \mathcal{P}_{\Omega_l^{(t)}}
    \left( \overline{\mathcal{G}}_{\mathrm{proj},l} \right),
\end{equation*}
where $\mathcal{P}_{\Omega}$ keeps entries indexed by $\Omega$ and sets all
others to zero. 

We alternate the support-constrained core update with orthogonal
factor-refitting steps. Each pruning stage uses two rounds of factor
refitting, after which the core is again projected onto the same support.
Thus, refinement may change the Tucker factors and the values of the
surviving coefficients, but it does not increase the number of active core
entries.

The refined representation is then used to determine the magnitude
ordering for the next support reduction.  The resulting procedure
therefore alternates
\begin{equation*}
    \text{support reduction}
    \;\longrightarrow\;
    \text{refit}^{(1)}
    \;\longrightarrow\;
    \text{refit}^{(2)}
    \;\longrightarrow\;
    \text{support reduction},
\end{equation*}
until
\begin{equation*}
    \left\| \overline{\mathcal{G}}_{l} \right\|_{0}
    = K_{l}^\star.
\end{equation*}
This gradual prune--refit schedule allows the shared Tucker subspaces and
the surviving core coefficients to adapt as the sparsity constraint is
progressively tightened, rather than fixing the final support from the
initial dense core in a single pruning step.

\subsection{Calibration and Compression of the Output Projection}

The attention output projection $W_{O,l}$ is treated separately from the Q/K/V projections not only because it receives a different input representation, but also because it occupies a fundamentally different position in the attention computation. The Q/K/V projections enter the
coupled nonlinear map
\begin{equation*}
    \mathcal{A}(Q,K,V)
    =\operatorname{softmax}
    \left(\frac{QK^{\top}}{\sqrt{d_h}}\right)V.
\end{equation*}
In contrast, after the attention output is formed, $W_{O,l}$ is applied through a linear transformation to the resulting representation.  At the output-projection stage, a candidate weight matrix $W$ produces the student output $Z_l^{S}W,$ whereas the corresponding teacher output is $Z_l^{T}W_{O,l}.$ Therefore, for fixed $Z_l^{S}$ and $Z_l^{T}$, the quantity
\begin{equation*}
    \frac{1}{N}
    \left\| Z_l^{S}W-Z_l^{T}W_{O,l} \right\|_{F}^{2}
\end{equation*}
directly quantifies the local discrepancy between the student and teacher outputs at the output-projection stage.

\paragraph{Teacher--student calibrated target.}
Because the upstream Q/K/V compression changes the representation entering
the output projection from $Z_l^{T}$ to $Z_l^{S}$, directly compressing the
pretrained $W_{O,l}$ ignores this representation shift. When
output-projection calibration is enabled, we then construct a
regularized weight target by solving
\begin{equation}
    \min_{W}\;\frac{1}{N}
    \left\| Z_{l}^{S}W-Z_{l}^{T}W_{O,l} \right\|_{F}^{2}
    +\mu_{O}\left\|W-W_{O,l}\right\|_{F}^{2},
    \label{eq:wo_calibration}
\end{equation}
where $\mu_{O}>0$ controls the deviation of the calibrated solution from the pretrained output projection. The first term seeks a linear operator whose output on the current student representation matches the corresponding teacher branch output, while the second term regularizes this
correction toward the original $W_{O,l}$. Thus, the calibration step compensates for the upstream representation shift without allowing an unconstrained departure from the pretrained projection. Define the empirical second moment and cross second moment
\begin{equation*}
    C_{s,l}^{O}=\frac{1}{N}(Z_{l}^{S})^{\top}Z_{l}^{S},
    \qquad
    C_{st,l}^{O}=\frac{1}{N}(Z_l^{S})^{\top}Z_l^{T}.
\end{equation*}
The first-order optimality condition of
\eqref{eq:wo_calibration} gives
\begin{equation*}
    \left(C_{s,l}^{O}+\mu_{O} I\right)W
    =\left(C_{st,l}^{O}+\mu_{O} I\right)W_{O,l},
\end{equation*}
and hence the calibrated target is
\begin{equation}
    W_{O,l}^{\mathrm{cal}}
    =\left(C_{s,l}^{O}+\mu_{O} I\right)^{-1}
    \left(C_{st,l}^{O}+\mu_{O} I\right) W_{O,l}.
    \label{eq:wo_calibrated_target}
\end{equation}
In practice, \eqref{eq:wo_calibrated_target} is evaluated through a linear solve rather than by explicitly forming the matrix inverse.

To cover both calibrated and non-calibrated configurations, we write
\begin{equation*}
    W_{O,l}^{\mathrm{tgt}}
    =\begin{cases}
        W_{O,l}^{\mathrm{cal}},
        & \text{if output-projection calibration is enabled},\\[2mm]
        W_{O,l},
        & \text{otherwise}.
    \end{cases}
\end{equation*}
Thus, the calibration step determines the weight matrix to be approximated,
whereas the approximation geometry is constructed separately below.

\paragraph{Output-specific activation-induced metric.}
The compression error of an approximation $\widehat{W}_{O,l}$ is measured under the input distribution of the current student branch. Let
\begin{equation*}
    \Delta W_{O,l}
    = W_{O,l}^{\mathrm{tgt}}-\widehat{W}_{O,l}.
\end{equation*}
Then
\begin{align*}
    \frac{1}{N}
    \left\|Z_{l}^{S}\Delta W_{O,l}\right\|_{F}^{2}
    &=\operatorname{tr}
    \left(\Delta W_{O,l}^{\top}C_{s,l}^{O}\Delta W_{O,l}
    \right).
\end{align*}
This identity induces a specific metric from the current
student representations. We denote by $\overline{C}_{l}^{O}$ the empirical second-moment matrix used to construct the output-projection metric, with $\overline{C}_{l}^{O}=C_{s,l}^{O}$ when only the current student representations are used.

\paragraph{Metric whitening and low-rank compression.}
Let $H_{\varepsilon,l}^{O}=(L_{l}^{O})^{\top}L_{l}^{O}$
be a Cholesky factorization. The weighted approximation objective can then be rewritten as
\begin{align*}
    \left\| W_{O,l}^{\mathrm{tgt}}
    -\widehat{W}_{O,l}\right\|_{H_{\varepsilon,l}^{O}}^{2}
    &=\operatorname{tr}\left[\left(
    W_{O,l}^{\mathrm{tgt}}-\widehat{W}_{O,l}\right)^{\top}
    H_{\varepsilon,l}^{O}\left(
            W_{O,l}^{\mathrm{tgt}}-\widehat{W}_{O,l}
        \right)\right]\nonumber\\
    &=\left\|L_{l}^{O} \left(W_{O,l}^{\mathrm{tgt}}
            -\widehat{W}_{O,l}\right)\right\|_{F}^{2}.
\end{align*}
Accordingly, we define the whitened target
\begin{equation*}
    \widetilde{W}_{O,l}^{\mathrm{tgt}}
    =L_{l}^{O}W_{O,l}^{\mathrm{tgt}}.
\end{equation*}
We compute its singular value decomposition
\begin{equation*}
    \widetilde{W}_{O,l}^{\mathrm{tgt}}
    =U_{l}^{O}\Sigma_{l}^{O}(V_{l}^{O})^{\top},
\end{equation*}
and retain the leading $r_{O,l}$ singular components, where $r_{O,l}$ is chosen according to the prescribed output-projection compression budget.
This gives
\begin{equation*}
    \widetilde{W}_{O,l}^{(r)}
    =U_{l,r}^{O}\Sigma_{l,r}^{O}(V_{l,r}^{O})^{\top}.
\end{equation*}
The compressed output projection is finally recovered in the original
weight coordinates by solving
\begin{equation*}
    L_l^{O}\widehat W_{O,l}
    =\widetilde W_{O,l}^{(r)}.
\end{equation*}
Thus, the output projection is compressed through an independent
metric-whitened low-rank approximation whose metric is induced by its
own attention-input representations.

\subsection{Optional Pattern-Aligned Refinement for 2:4 Sparsity}

For hard structured sparsity such as $2{:}4$, the pruning result depends on the coordinate representation of the Tucker core. To improve the compatibility between the dense Tucker representation and the prescribed sparsity pattern, we optionally apply an orthogonal gauge transformation
before structured pruning.  Recall that the compressed Q/K/V tensor is represented as
\begin{equation*}
    \widehat{\mathcal{W}}_{l}
    = \mathcal{G}_{l}
    \times_1 U_{1,l}
    \times_2 U_{2,l}.
\end{equation*}
For any orthogonal matrices
$Q_{1,l}^{\top}Q_{1,l}=I,\quad Q_{2,l}^{\top}Q_{2,l}=I,$
we may equivalently write
\begin{equation*}
    \widehat{\mathcal{W}}_{l}
    = \mathcal{G}_{l,Q}
    \times_1 U_{1,l}^{Q}
    \times_2 U_{2,l}^{Q},
\end{equation*}
where
\begin{equation*}
    \mathcal{G}_{l,Q}=\mathcal{G}_{l}
    \times_1 Q_{1,l}^{\top}\times_2 Q_{2,l}^{\top},
    \qquad
    U_{1,l}^{Q} =U_{1,l}Q_{1,l},
    \qquad
    U_{2,l}^{Q}=U_{2,l}Q_{2,l}.
\end{equation*}
Thus, the rotation leaves the dense Tucker reconstruction unchanged, while redistributing the coefficients within the latent core. 

Let $\mathcal{P}_{2:4}$ denote the projection that enforces the prescribed 2:4 pattern on the core. We seek an orthogonal rotation that approximately reduces the structured projection error,
\begin{equation}
    \min_{Q_{1,l},Q_{2,l}}
    \left\|\mathcal{G}_{l,Q}- \mathcal{P}_{2:4}
        \left(\mathcal{G}_{l,Q}\right)\right\|_{F}^{2},
    \qquad
    Q_{1,l}^{\top}Q_{1,l}=Q_{2,l}^{\top}Q_{2,l}=I.
    \label{eq:pattern_alignment}
\end{equation}
Since the rotations are orthogonal,
$\|\mathcal G_{l,Q}\|_F=\|\mathcal G_l\|_F$, and
\eqref{eq:pattern_alignment} can equivalently be viewed as concentrating
more core energy on the coefficients that survive the 2:4 projection.

The resulting rotation is used only when it provides a sufficient increase
in the retained core energy. Specifically, letting
\begin{equation*}
    \eta_{2:4}(\mathcal G)
    = \frac{\left
        \|\mathcal{P}_{2:4}(\mathcal{G})\right\|_{F}^{2}
    }{\|\mathcal{G}\|_{F}^{2}},
\end{equation*}
we apply the rotated representation only if
\begin{equation}
    \eta_{2:4}(\mathcal{G}_{l,Q})
    -\eta_{2:4}(\mathcal{G}_{l})>\tau.
\end{equation}
Otherwise, the original Tucker coordinates are retained. The accepted representation is then used as the initialization for the subsequent 2:4-constrained sparse-core refinement.

This procedure does not alter the Tucker ranks or the dense represented subspace. Rather, it selects a latent coordinate system in which the prescribed 2:4 constraint removes less energy before structured refinement. Consequently, the rotation is used as a guarded pattern-alignment step rather than as an additional source of compression.

\subsection{Numerical Stability and Practical Details}
\label{app:practical_details}

Since the activation-induced statistics are estimated from finite
calibration data, the resulting Gram matrices may be singular or poorly
conditioned. We therefore apply diagonal damping to both the Q/K/V and
output-projection metrics before whitening. The damping parameter
$\varepsilon$ is used only for numerical stabilization and is distinct
from the regularization parameter $\mu_O$ used to construct the calibrated
target for $W_{O,l}$.

\paragraph{Teacher-Gram anchoring for $W_O$.}
For the output projection, we additionally anchor the student attention
statistics with the corresponding teacher statistics. Let
\begin{equation*}
    C_{s,l}^{O} = \frac{1}{N}(Z_l^{\mathrm{S}})^{\top}Z_{l}^{\mathrm{S}},
    \qquad
    C_{t,l}^{O} = \frac{1}{N}(Z_{l}^{\mathrm T})^{\top}Z_{l}^{\mathrm{T}}.
\end{equation*}
We form the anchored Gram matrix as
\begin{equation*}
    C_{\mathrm{anc},l}^{O} = (1-a)C_{s,l}^{O} + a C_{t,l}^{O},
\end{equation*}
with $a=0.5$ in the final recipe. This anchoring is used only for $W_O$;
the Q/K/V metric remains unanchored and is constructed solely from the
current student representations. Both the Q/K/V and output-projection metrics are damped before whitening using the same damping coefficient $\varepsilon$. 

\paragraph{Stable metric transformations.}
For a damped metric
\begin{equation}
    H_{\varepsilon,l}=L_l^{\top}L_l,
\end{equation}
we perform whitening and inverse whitening through linear or triangular
solves rather than explicitly forming $H_{\varepsilon,l}^{-1}$ or
$L_l^{-1}$. The same procedure is used for the output-projection metric.

\paragraph{Adaptive damping and fallback.}
The default damping is used whenever the resulting compressed
representation passes the stability check. If numerical instability is
detected, we increase the damping geometrically,
\begin{equation}
    \varepsilon_l^{(r+1)}
    =
    10\,\varepsilon_l^{(r)},
    \label{eq:damping_escalation}
\end{equation}
and recompute the compression. If the metric-based approximation remains
unstable, we use the identity metric as a final fallback.

All retries preserve the original compression target and parameter budget.
In particular, the Tucker ranks and sparse-core budget for Q/K/V, as well
as the low-rank budget for $W_{O,l}$, are kept unchanged. Therefore, the
stability safeguards improve numerical robustness without relaxing the
prescribed compression ratio or leaving unstable layers uncompressed.

\section{Additional Details for Sequential Activation-Induced Tensor Approximation}

This appendix provides additional details for the sequential
activation-induced tensor approximation introduced in
Section~\ref{sec:sequential_tensor_approximation}. We first derive the local activation-induced metric, then describe the construction of the joint head-wise Q/K/V tensor and its metric whitening, and finally clarify how the compression proceeds sequentially across Transformer layers.

\paragraph{Sequential student representations.}
The compression is performed layer by layer. At layer $l$, the calibration input $X_l^S$ is obtained from the current
student model with all preceding compressed layers already in place. Consequently, the metric at layer $l$ is constructed from the actual student representations and naturally incorporates the representation shift induced by earlier compression.

For a projection matrix $W_{t,l}$, the perturbation
$\Delta W_{t,l}=W_{t,l}-\widehat W_{t,l}$ induces the following local functional distortion on the current student activations:
\begin{equation}
    \frac{1}{N}\left\|X_{l}^{S} \Delta W_{t,l}\right\|_{F}^{2}.
    \label{eq:app_local_projection_error}
\end{equation}
This quantity measures the effect of the weight perturbation on the projection outputs over the calibration representations at layer $l$, and naturally induces a data-dependent  metric for evaluating the approximation error.

\paragraph{Activation-induced approximation metric.}
Define the empirical second-moment matrix
\begin{equation*}
    C_{s,l}
    =\frac{1}{N}(X_{l}^S)^\top X_{l}^S.
\end{equation*}
Then the local projection error in
\eqref{eq:app_local_projection_error} can be written as
\begin{align*}
    \frac{1}{N}
    \left\|X_{l}^S \Delta W_{t,l}\right\|_{F}^{2}
    &=\frac{1}{N}\operatorname{tr}
    \left(\Delta W_{t,l}^{\top}(X_{l}^S)^\top X_{l}^S
        \Delta W_{t,l}\right) \nonumber\\
    &=\operatorname{tr}
    \left(\Delta W_{t,l}^{\top}C_{s,l}\Delta W_{t,l}\right).
\end{align*}
We therefore define
\begin{equation*}
    \|\Delta W\|_{C_{s,l}}^{2} :=\operatorname{tr}
    \left(\Delta W^\top C_{s,l}\Delta W\right).
\end{equation*}
Hence, the activation-induced metric is directly obtained from the projection-output error evaluated on the current student representations.

Since the Q, K, and V projections share the same input $X_l^S$, they also share the same input-side metric $C_{s,l}$. 

\paragraph{Joint head-wise Q/K/V tensor construction.}
We reshape the Q/K/V projection matrices into their native head-wise blocks. Let
\begin{equation*}
    W_{Q,l}^{(i)}\in
    \mathbb{R}^{d_{\mathrm{model}}\times d_{\mathrm{head}}},
    \qquad i=1,\ldots,h_{q,l},
\end{equation*}
denote the query-head blocks, and similarly let
\begin{equation*}
    W_{K,l}^{(i)},\, W_{V,l}^{(i)} \in
    \mathbb{R}^{d_{\mathrm{model}}\times d_{\mathrm{head}}},
    \qquad i=1,\ldots,h_{kv,l},
\end{equation*}
denote the key and value blocks. We stack these matrices as frontal slices of a joint tensor
\begin{equation*}
    \mathcal{W}_{l}\in
    \mathbb{R}^{d_{\mathrm{model}}\times
        d_{\mathrm{head}}\times m_{l}},
    \qquad m_{l}=h_{q,l}+2h_{kv,l}.
\end{equation*}
The query slices are placed first, followed by the key and value slices.

The joint tensor retains the native head organization of the attention module. In GQA and MQA architectures, the shared key/value heads are kept in their original form rather than replicated to match the number of query heads. Accordingly, the third mode serves only to enumerate the native
Q/K/V head slices and is not compressed in the Tucker decomposition.

For a tensor
$\mathcal{A}\in\mathbb{R}^{d_{\mathrm{model}}\times d_{\mathrm{head}}\times m_{l}}$,
we extend the activation-induced metric as
\begin{equation*}
    \|\mathcal{A}\|_{C_{s,l}}^{2} :=
    \sum_{j=1}^{m_l}\operatorname{tr}
    \left( A_{j}^\top C_{s,l}A_{j}\right),
\end{equation*}
where $A_{j}$ denotes the $j$-th frontal slice. Thus, the joint tensor objective retains the same activation-induced metric for every Q/K/V slice while allowing their structural information to be modeled jointly.

\paragraph{Metric whitening.}
For numerical stability, we use the damped metric 
\begin{equation*}
    H_{\varepsilon,l} = C_{s,l} + \varepsilon_{l}\bar{t}_{l} I,
    \qquad\varepsilon_{l}>0.
\end{equation*}
Let $H_{\varepsilon,l} = L_{l}^\top L_{l}$
be its Cholesky factorization. Since the metric acts on the shared input dimension, it affects only mode-$1$ of the joint tensor. For any tensor $\mathcal{A}$ of the same size as $\mathcal{W}_{l}$,
\begin{align*}
    \|\mathcal{A}\|_{H_{\varepsilon,l}}^{2}
    &= \sum_{j=1}^{m_l}\operatorname{tr}
    \left( A_{j}^\top H_{\varepsilon,l} A_{j}\right)
    \nonumber\\ &=
    \left\| \mathcal{A}\times_1 L_{l}\right\|_{F}^{2}.
\end{align*}
Accordingly, we define the whitened Q/K/V tensor as
\begin{equation*}
    \widetilde{\mathcal{W}}_{l}
    =\mathcal{W}_{l}\times_1 L_{l}.
\end{equation*}
The activation-weighted tensor approximation can therefore be carried out
as an ordinary Frobenius-norm approximation in the whitened coordinates.

\paragraph{Sequential layer update.}
After obtaining the compressed Q/K/V projections
\begin{equation*}
    \widehat{W}_{Q,l}, \qquad \widehat{W}_{K,l},
    \qquad \widehat{W}_{V,l},
\end{equation*}
we replace their original counterparts in the student model. Subsequent representations at layer $l$ are evaluated using this updated student, so that the remaining compression operations are conditioned on the
current compressed state of the model.

After completing the compression of layer $l$, the updated student model is propagated forward to obtain the input representation $X_{l+1}^{S}$ for the next layer. The resulting sequential procedure can be summarized as
\begin{equation*}
    X_{l}^{S}\;\longrightarrow\;
    C_{s,l}\;\longrightarrow\;
    \widehat{\mathcal{W}}_{l}
    \;\longrightarrow\;\text{update student}
    \;\longrightarrow\; X_{l+1}^{S}.
\end{equation*}
Accordingly, the activation-induced metric is recomputed at each layer from the current student representations, thereby accounting for the representation changes introduced by the preceding compression steps.

\section{Proof of the Weighted Shared-Subspace Characterization}
\label{app:weighted_shared_subspace}

We provide the proof of the weighted shared mode-$1$ subspace
characterization stated in Section~\ref{sec:shared_tensor_representation}.
We follow the notation introduced there.

\begin{proposition}[Weighted shared mode-$1$ subspace]
\label{prop:weighted_shared_mode1}
Consider
\begin{equation}
    \min_{\substack{
        \widetilde{U}_{l}^{\top}\widetilde{U}_{l}=I\\
        B_{1,l},\ldots,B_{m_l,l}}}
    \sum_{j=1}^{m_l} w_{j,l}
    \left\|\widetilde{W}_{l,j}
        -\widetilde{U}_{l} B_{j,l}\right\|_{F}^{2},
    \label{eq:app_weighted_shared_mode1}
\end{equation}
where $w_{j,l}>0$ and $\widetilde{U}_{l}\in
\mathbb{R}^{d_{\mathrm{model}}\times R_{1,l}}$.
For any fixed $\widetilde{U}_{l}$, the optimal coefficient matrices are
\begin{equation*}
    B_{j,l}^{\star}
    =\widetilde{U}_{l}^{\top}\widetilde W_{l,j}.
\end{equation*}
Moreover, an optimal $\widetilde{U}_{l}$ is given by the leading
$R_{1,l}$ left singular vectors of
\begin{equation}
    \widetilde{W}_{\mathrm{cat},l}
    =\begin{bmatrix}
        \sqrt{w_{1,l}}\widetilde{W}_{l,1}
        &\cdots &\sqrt{w_{m_l,l}}\widetilde W_{l,m_l}
    \end{bmatrix}.
    \label{eq:app_weighted_concat}
\end{equation}
\end{proposition}

\begin{proof}
For fixed $\widetilde{U}_{l}$, the orthogonality condition
$\widetilde{U}_{l}^{\top}\widetilde{U}_{l}=I$ gives
\begin{align*}
    \left\|\widetilde{W}_{l,j}
        -\widetilde U_l B_{j,l}\right\|_{F}^{2}
    &=\|\widetilde{W}_{l,j}\|_{F}^{2}
    - 2\left\langle
        \widetilde{U}_{l}^{\top}\widetilde{W}_{l,j},B_{j,l}
    \right\rangle +\|B_{j,l}\|_{F}^{2}\nonumber\\
    &= \|\widetilde{W}_{l,j}\|_{F}^{2}-
    \left\|\widetilde{U}_{l}^{\top}\widetilde{W}_{l,j}\right\|_{F}^{2}
    +\left\| B_{j,l} -
        \widetilde{U}_{l}^{\top}\widetilde{W}_{l,j}\right\|_{F}^{2}.
\end{align*}
Therefore,
\begin{equation*}
    B_{j,l}^{\star} = \widetilde{U}_{l}^{\top}\widetilde{W}_{l,j}.
\end{equation*}
Substituting these minimizers into
\eqref{eq:app_weighted_shared_mode1} yields
\begin{align*}
    & \sum_{j=1}^{m_l} w_{j,l}
    \left\| \widetilde{W}_{l,j} - \widetilde{U}_l B_{j,l}^{\star}
    \right\|_{F}^{2} \nonumber\\
    &= \sum_{j=1}^{m_l} w_{j,l} \|\widetilde{W}_{l,j}\|_{F}^{2}
    -\sum_{j=1}^{m_l} w_{j,l}
    \left\| \widetilde{U}_{l}^{\top}\widetilde{W}_{l,j}
    \right\|_{F}^{2}.
\end{align*}
The first term is independent of $\widetilde{U}_{l}$. Therefore, the
problem is equivalent to
\begin{equation*}
    \max_{\widetilde{U}_{l}^{\top}\widetilde{U}_{l}=I}
    \sum_{j=1}^{m_l} w_{j,l}
    \left\| \widetilde{U}_{l}^{\top}\widetilde{W}_{l,j}\right\|_{F}^{2}.
\end{equation*}
Using the weighted concatenation in
\eqref{eq:app_weighted_concat}, we have
\begin{equation*}
    \sum_{j=1}^{m_l}  w_{j,l}
    \left\| \widetilde{U}_{l}^{\top}\widetilde{W}_{l,j} \right\|_{F}^{2}
    = \left\| \widetilde{U}_l^{\top} \widetilde{W}_{\mathrm{cat},l}
    \right\|_{F}^{2}.
\end{equation*}
Equivalently,
\begin{equation*}
    \left\|
        \widetilde{U}_l^{\top}\widetilde{W}_{\mathrm{cat},l}
    \right\|_{F}^{2}=\operatorname{tr}
    \left( \widetilde{U}_l^{\top}
        \widetilde{W}_{\mathrm{cat},l}
        \widetilde{W}_{\mathrm{cat},l}^{\top}
        \widetilde{U}_{l} \right).
\end{equation*}
Maximizing this trace over orthonormal $\widetilde{U}_l$ selects the
dominant $R_{1,l}$-dimensional eigenspace of $\widetilde{W}_{\mathrm{cat},l}
\widetilde{W}_{\mathrm{cat},l}^{\top}$, equivalently the leading left
singular subspace of $\widetilde{W}_{\mathrm{cat},l}$.
\end{proof}

This result formally justifies the weighted concatenation used to identify
the shared mode-$1$ subspace. The full Tucker representation in
Section~\ref{sec:shared_tensor_representation} extends the same sharing
principle to both compressed modes, while retaining the third mode in full.

\section{Additional Details for Budget-Coupled Rank--Sparsity Allocation}
\label{app:rank_sparsity}

This section provides the supporting derivations for
Proposition~\ref{prop:budget_allocation}. We first establish the
fixed-subspace error decomposition and the optimal top-$K_l$ core projection,
and then show how these results lead to the budget-coupled rank search.
The gradual pruning and support-constrained refinement used to obtain the
final sparse core are described separately in
Appendix~\ref{app:core_refinement}.

\subsection{Fixed-Subspace Error Decomposition}

For a candidate rank pair $\mathbf{R}_{l}=(R_{1,l},R_{2,l})$, let
$\widetilde{U}_{1,l}$ and $U_{2,l}$ be the corresponding orthonormal
factors, and recall the projection core
\begin{equation*}
    \overline{\mathcal{G}}_l^{\star}
    =\overline{\mathcal{W}}_{l}\times_1 \widetilde{U}_{1,l}^{\top}
    \times_2 U_{2,l}^{\top}.
\end{equation*}
Define the associated projection residual by
\begin{equation}
    \mathcal{E}_{l}:=\overline{\mathcal{W}}_{l}
    -\overline{\mathcal{G}}_{l}^{\star}
    \times_1 \widetilde{U}_{1,l}\times_2 U_{2,l}.
    \label{eq:app_projection_residual}
\end{equation}

\begin{proposition}[Fixed-subspace error decomposition]
\label{prop:fixed_subspace_decomposition}
For any $\mathcal{G}\in\mathbb{R}^{R_{1,l}\times R_{2,l}\times m_l}$,
\begin{align}
    &\left\|\overline{\mathcal{W}}_{l}-\mathcal{G}
        \times_1 \widetilde{U}_{1,l}\times_2 U_{2,l}
    \right\|_{F}^{2}\nonumber\\
    &\qquad =\|\mathcal{E}_{l}\|_{F}^{2}
    +\left\|\overline{\mathcal{G}}_l^{\star}-\mathcal{G}\right\|_{F}^{2}.
    \label{eq:app_fixed_subspace_identity}
\end{align}
\end{proposition}

\begin{proof}
Using \eqref{eq:app_projection_residual}, we have
\begin{align}
    &\overline{\mathcal{W}}_{l}-\mathcal{G}
    \times_1 \widetilde{U}_{1,l}\times_2 U_{2,l}\nonumber\\
    &=\mathcal{E}_{l}+
    \left(\overline{\mathcal G}_l^{\star}-\mathcal{G}\right)
    \times_1 \widetilde{U}_{1,l}
    \times_2 U_{2,l}.
    \label{eq:app_error_split}
\end{align}
It observes that $\mathcal{E}_{l}$ is the residual of the orthogonal
projection onto the Tucker subspace generated by
$\widetilde{U}_{1,l}$ and $U_{2,l}$. Hence,
\begin{equation*}
    \left\langle \mathcal{E}_{l},\, \mathcal{A}
        \times_1 \widetilde{U}_{1,l}\times_2 U_{2,l} \right\rangle =0
\end{equation*}
for any compatible tensor $\mathcal{A}$.

Therefore, the two terms in \eqref{eq:app_error_split} are orthogonal.
Since the two factors have orthonormal columns,
\begin{equation*}
    \left\| \mathcal{A}
        \times_1 \widetilde{U}_{1,l}\times_2 U_{2,l}
    \right\|_{F}=\|\mathcal{A}\|_{F}.
\end{equation*}
It follows that
\begin{align*}
    & \left\|\overline{\mathcal{W}}_{l}-\mathcal{G}
        \times_1 \widetilde{U}_{1,l}\times_2 U_{2,l}
    \right\|_{F}^{2}\nonumber\\
    &=\|\mathcal{E}_{l}\|_{F}^{2} +\left\|
        \overline{\mathcal{G}}_l^{\star}-\mathcal{G}\right\|_{F}^{2}.
\end{align*}
The quantity $\|\mathcal{E}_{l}\|_{F}^{2}$ represents the approximation error due to the selected Tucker subspace, which we denote by
$E_{\mathrm{sub}}(\mathbf{R}_{l})$ in the following.
\end{proof}

Proposition~\ref{prop:fixed_subspace_decomposition} shows that the total
approximation error separates into the error of the selected Tucker
subspace and the error incurred when approximating the corresponding
projection core. In particular, the latter can be evaluated directly in
the latent core space.

\subsection{Optimal Cardinality Projection of the Core}

We next specialize the previous result to the entry-wise cardinality
constraint used for candidate scoring.

\begin{proposition}
For fixed orthonormal factors $\widetilde{U}_{1,l}$ and $U_{2,l}$, the solution of
\begin{equation}
    \min_{\|\mathcal{G}\|_0\leq K_{l}}
    \left\| \overline{\mathcal{W}}_{l} - \mathcal{G}
        \times_1 \widetilde{U}_{1,l} \times_2 U_{2,l}
    \right\|_{F}^{2}
    \label{eq:app_sparse_core_problem}
\end{equation}
is
\begin{equation}
    \mathcal{G}^{\star}_{K_{l}} = \mathcal{P}_{K_{l}}
    \left( \overline{\mathcal{G}}_l^{\star} \right),
    \label{eq:app_topk_solution}
\end{equation}
where $\mathcal{P}_{K_{l}}$ retains the $K_{l}$ largest-magnitude entries of
the projection core. Consequently,
\begin{equation}
    E_{\mathrm{struct}}(\mathbf{R}_{l},K_{l})
    = E_{\mathrm{sub}}(\mathbf{R}_{l})
    +  D_{\mathrm{sp}}(\mathbf{R}_l,K_{l}).
    \label{eq:app_rank_sparsity_decomposition}
\end{equation}
\end{proposition}

\begin{proof}
By Proposition~\ref{prop:fixed_subspace_decomposition},
problem~\eqref{eq:app_sparse_core_problem} is equivalent to
\begin{equation*}
    \min_{\|\mathcal{G}\|_0\leq K_{l}}
    \left\| \overline{\mathcal{G}}_l^{\star} - \mathcal{G}
    \right\|_{F}^{2},
\end{equation*}
since the subspace error does not depend on $\mathcal{G}$. The best
$K_l$-term approximation in the Frobenius norm is obtained by retaining
the $K_l$ entries of largest magnitude. Substituting
\eqref{eq:app_topk_solution} into
\eqref{eq:app_fixed_subspace_identity} yields
\eqref{eq:app_rank_sparsity_decomposition}.
\end{proof}

Eq.~\eqref{eq:app_rank_sparsity_decomposition} is exact for the
selected Tucker subspace, with the factors and ranks fixed and
$\overline{\mathcal{G}}_{l}^{\star}$ taken as the corresponding projection
core. Subsequent refinement may further adjust the core values, support,
and Tucker factors.

\subsection{Interpretation of Latent-Core Sparsity}

The sparsity constraint in FTC is imposed on the entries of the latent
Tucker core rather than directly on the reconstructed Q/K/V weight
matrices. To make this distinction explicit, write
\begin{equation*}
    \widetilde{U}_{1,l} = \begin{bmatrix}
        \widetilde{u}_{1,l}^{(1)} & \cdots &\widetilde{u}_{1,l}^{(R_{1,l})}
    \end{bmatrix},
    \qquad
    U_{2,l}
    =\begin{bmatrix}
        u_{2,l}^{(1)} & \cdots &  u_{2,l}^{(R_{2,l})}
    \end{bmatrix}.
\end{equation*}
For the $j$-th frontal slice, the reconstructed weight matrix is
\begin{equation*}
    \widehat{\overline{W}}_{l,j} =
    \widetilde{U}_{1,l} G_{l,j}  U_{2,l}^{\top},
\end{equation*}
where $G_{l,j} = \mathcal{G}(:,:,j)$. Then expanding the matrix product gives
\begin{equation}
    \widehat{\overline{W}}_{l,j} = \sum_{r_1=1}^{R_{1,l}}
    \sum_{r_2=1}^{R_{2,l}} g_{r_1,r_2,j}\,
    \widetilde{u}_{1,l}^{(r_1)}
    \left( u_{2,l}^{(r_2)} \right)^{\top},
    \label{eq:app_rankone_expansion}
\end{equation}
where $g_{r_1,r_2,j}$ denotes the $(r_1,r_2)$-th entry of the
$j$-th core slice.

Eq.~\eqref{eq:app_rankone_expansion} provides a structural
interpretation of core sparsity. Each core coefficient controls one
rank-one interaction between a mode-$1$ basis vector
$\widetilde{u}_{1,l}^{(r_1)}$ and a mode-$2$ basis vector
$u_{2,l}^{(r_2)}$. Therefore, setting a core coefficient to zero removes
the corresponding rank-one component from the factorized representation.

Let $ \mathcal{S}_{l,j} = \left\{ (r_1,r_2):
        g_{r_1,r_2,j}\neq 0 \right\}$
denote the active support of the $j$-th core slice. Then
\begin{equation}
    \widehat{\overline{W}}_{l,j}
    = \sum_{(r_1,r_2)\in\mathcal{S}_{l,j}}
    g_{r_1,r_2,j}\, \widetilde{u}_{1,l}^{(r_1)}
    \left( u_{2,l}^{(r_2)} \right)^{\top}.
    \label{eq:app_sparse_rankone_expansion}
\end{equation}
Hence, the sparse core determines which rank-one interactions
between the shared factor directions are active for each Q/K/V projection
slice. The Tucker factors are shared across all Q/K/V slices, while each slice has its own core coefficients and may activate a different subset of the resulting rank-one components.

This view also clarifies the difference between latent-core sparsity and
entry-wise weight sparsity. Although only a subset of the rank-one components is active in \eqref{eq:app_sparse_rankone_expansion}, each rank-one component is generally dense because both factor vectors are dense. Consequently, $\widehat{\overline{W}}_{l,j}$ is typically a dense matrix even when $G_{l,j}$ is highly sparse. The sparsity therefore lies in the
low-dimensional interaction structure between the two shared factor
spaces, rather than in the entries of the reconstructed weight matrix.

Accordingly, the computational form of the compressed representation is
not a materialized dense matrix but the sparse rank-one expansion in
\eqref{eq:app_sparse_rankone_expansion}. Computational benefits from core
sparsity therefore require an implementation that exploits this sparse
latent interaction structure instead of first reconstructing the dense
Q/K/V weight matrices.

This interpretation is also consistent with the parameter accounting in
Section~\ref{sec:budget_rank_sparsity}. The shared factors contribute
$d_{\mathrm{model}}R_{1,l}+d_{\mathrm{head}}R_{2,l}$ dense parameters,
while the sparse core contributes the retained latent coefficients as well as the storage required to encode their support.
Thus, $K_l$ measures the number of active rank-one interactions across the
joint Q/K/V representation, rather than the number of nonzero entries in
the reconstructed projection matrices.

\subsection{Efficient Evaluation of Candidate Rank Pairs}

The rank--sparsity allocation requires comparing multiple feasible
rank pairs, but running a complete Tucker optimization and sparse-core
refinement for every candidate would introduce substantial overhead.
We therefore use a common HOSVD-based initialization to evaluate all
candidate configurations efficiently. The key observation is that the
candidate Tucker subspaces are nested and can be extracted from a single
pair of mode-wise singular subspaces.

Let $\overline{\mathcal{W}}_{l,(1)}$ and
$\overline{\mathcal{W}}_{l,(2)}$ denote the mode-$1$ and mode-$2$
unfoldings of the balanced whitened tensor $\overline{\mathcal{W}}_l$. We compute
\begin{equation*}
    \overline{\mathcal{W}}_{l,(1)}
    =  \widetilde{U}_{1,l}^{\mathrm{full}}
    \Sigma_{1,l} V_{1,l}^{\top},  \qquad  \overline{\mathcal{W}}_{l,(2)}
    =  U_{2,l}^{\mathrm{full}}  \Sigma_{2,l} V_{2,l}^{\top}.
\end{equation*}
For a candidate $\mathbf{R}_{l}=(R_{1,l},R_{2,l})$, the corresponding Tucker factors are obtained by truncating these common bases,
\begin{equation*}
    \widetilde{U}_{1,l}^{(\mathbf{R}_{l})}
    = \widetilde{U}_{1,l}^{\mathrm{full}}(:,1:R_{1,l}), \qquad
    U_{2,l}^{(\mathbf{R}_{l})} = U_{2,l}^{\mathrm{full}}(:,1:R_{2,l}).
\end{equation*}
Hence, the mode-wise singular subspaces are computed only once for each
layer and reused across all candidate rank pairs.

Let $R_{1,l}^{\max}$ and $R_{2,l}^{\max}$ denote the largest ranks
appearing in the candidate grid. We further form the maximal projection
core
\begin{equation*}
    \overline{\mathcal{G}}_{l}^{\max} =
    \overline{\mathcal{W}}_{l} \times_1
    \left(  \widetilde{U}_{1,l}^{\mathrm{full}}  (:,1:R_{1,l}^{\max}) \right)^{\top}
    \times_{2}
    \left( U_{2,l}^{\mathrm{full}} (:,1:R_{2,l}^{\max})\right)^{\top},
\end{equation*}
its projection core is 
\begin{equation}
    \overline{\mathcal{G}}_{l}^{\star}(\mathbf{R}_{l})
    =  \overline{\mathcal{G}}_{l}^{\max} (1:R_{1,l},\,1:R_{2,l},\,:).
    \label{eq:app_candidate_core_slice}
\end{equation}
Thus, the projection core can be reused across different candidate rank pairs.

The subspace approximation error can also be evaluated without explicitly
reconstructing the Tucker approximation. It can be written as 
\begin{equation}
    E_{\mathrm{sub}}(\mathbf{R}_{l})
    = \|\overline{\mathcal{W}}_{l}\|_{F}^{2}
    - \left\|
        \overline{\mathcal{G}}_{l}^{\star}(\mathbf{R}_{l})
    \right\|_F^2.
    \label{eq:app_subspace_energy}
\end{equation}
Therefore, candidate evaluation requires only the energy of the
corresponding projection core rather than an explicit reconstruction of
$\overline{\mathcal{W}}_{l}$.

For the same candidate, the layerwise budget determines the admissible
core cardinality $K_l(\mathbf{R}_l)$. Let 
$|g_{l,(1)}| \geq  |g_{l,(2)}| \geq  \cdots$ 
denote the magnitudes of the entries of
$\overline{\mathcal G}_{l}^{\star}(\mathbf R_l)$ 
arranged in descending order. The one-step sparsification error can then be evaluated directly from the discarded core energy:
\begin{equation}
    D_{\mathrm{sp}} \left(
        \mathbf{R}_{l}, K_l(\mathbf{R}_{l}) \right)
    =\left\|\overline{\mathcal{G}}_{l}^{\star}(\mathbf{R}_{l})
    \right\|_{F}^{2}
    - \sum_{q=1}^{K_l(\mathbf{R}_{l})}|g_{l,(q)}|^2.
    \label{eq:app_sparse_energy}
\end{equation}

Combining Eqs.~\eqref{eq:app_subspace_energy} and
\eqref{eq:app_sparse_energy}, each feasible rank pair is evaluated using
only quantities extracted from the common singular subspaces and the
corresponding projection core. The candidate score is
\begin{equation*}
    J_l(\mathbf{R}_{l})
    =  E_{\mathrm{sub}}(\mathbf{R}_{l})
    + \beta D_{\mathrm{sp}}
    \left( \mathbf{R}_{l},  K_l(\mathbf{R}_{l})\right).
\end{equation*}

In our implementation, the candidate set is a predefined
$12\times4$ grid over the mode-$1$ and mode-$2$ ranks. Rank pairs whose factor and support metadata costs already exceed the layerwise budget are discarded before scoring. The remaining candidates share the same mode-wise
decompositions, and their projection cores are obtained from the nested
subtensors in \eqref{eq:app_candidate_core_slice}. Consequently, the
allocation stage requires one pair of mode-wise decompositions and one
maximal core projection per layer, followed by inexpensive energy and
top-$K_l$ evaluations over the candidate grid.

The resulting top-$K_l$ projection is used only as a proxy for comparing
rank configurations. After selecting $\mathbf{R}_{l}^\star$, the corresponding Tucker representation is passed to the gradual sparse-core refinement procedure, where the core is progressively sparsified and refitted rather than being directly fixed to the one-step top-$K_l^\star$ projection.

\end{document}